\documentclass[times,final]{elsarticle}

\usepackage{jcomp}
\usepackage{framed,multirow}

\usepackage{amsmath}
\usepackage{amsthm}
\usepackage{amssymb}
\usepackage{latexsym}
\usepackage{amsopn}
\usepackage{url}
\usepackage{lipsum}
\usepackage{amsfonts}
\usepackage{graphicx}
\usepackage{subfigure}
\usepackage{epstopdf}
\usepackage{algorithmic}
\usepackage{algorithm}
\usepackage{comment}
\usepackage{xcolor}
\usepackage[normalem]{ulem}

\ifpdf
  \DeclareGraphicsExtensions{.eps,.pdf,.png,.jpg}
\else
  \DeclareGraphicsExtensions{.eps}
\fi

\newcommand{\RR}{\mathbb{R}} 
\newcommand{\XX}{\mathcal{X}} 
\newcommand {\inp}[1]{\left\langle #1 \right\rangle} 
\newtheorem{theorem}{Theorem}

\newtheorem{lemma}{Lemma}
\newtheorem{proposition}{Proposition}

\newtheorem{cor}{Corollary}
\newcommand{\kl}[1]{{\color{red}Kamel: {#1}}}

\definecolor{newcolor}{rgb}{.8,.349,.1}

\journal{Journal of Computational Physics}

\begin{document}

\verso{Kamel Lahouel \textit{etal}}

\begin{frontmatter}

\title{Learning nonparametric ordinary differential equations from noisy data}%

\author[1]{Kamel \snm{Lahouel}}
\ead{klahouel@tgen.org}
\author[2]{Michael \snm{Wells}}
\ead{mlwells@pdx.edu}
\author[2]{Victor \snm{Rielly}}
\ead{victor23@pdx.edu}
\author[3]{Ethan \snm{Lew}}
\ead{elew@galois.com}
\author[2]{David \snm{Lovitz}}
\ead{lovitz@pdx.edu}
\author[2]{Bruno M. \snm{Jedynak}\corref{cor1}}
\ead{bruno.jedynak@pdx.edu}
\cortext[cor1]{Corresponding author: 
  Tel.: +1-503-725-8283;  }
\address[2]{Dept. of Math \& Stat, Portland State University, 1855 SW Broadway, Portland, OR 97201 }
\address[1]{TGen, 445 N. Fifth Street, Phoenix, AZ 85004 }
\address[3]{Galois Inc., 421 SW 6th Avenue, Suite 300, Portland, Oregon 97204  }

\received{Submission}

\begin{abstract}
Learning nonparametric systems of Ordinary Differential Equations (ODEs) $\dot x = f(t,x)$ from noisy data is an emerging machine learning topic. We use the well-developed theory of Reproducing Kernel Hilbert Spaces (RKHS) to define candidates for $f$ for which the solution of the ODE exists and is unique. Learning $f$ consists of solving a constrained optimization problem in an RKHS. We propose a penalty method that iteratively uses the Representer theorem and Euler approximations to provide a numerical solution. We prove a generalization bound for the $L^2$ distance between $x$ and its estimator.  Experiments are provided for the FitzHugh–Nagumo oscillator, the Lorenz system, and for predicting the Amyloid level in the cortex of aging subjects.  In all cases, we show competitive results compared with the state-of-the-art. 
\end{abstract}

\end{frontmatter}

\section{Introduction}

\subsection{Description of the problem and related works}
Fitting a system of nonparametric ordinary differential equations (ODEs) $\dot x = f(t,x)$ to longitudinal data could lead to scientific breakthroughs in disciplines where ODEs or dynamical systems have been used for a long time, including physics, chemistry, and biology, see \cite{hirsch2012differential}. By nonparametric, we mean that there is no need to specify the functional form of the vector-field $f$ using a {pre-defined} finite dimensional parameter. Instead, this force field belongs to a functional space {and the number of parameters that characterize this vector field depends on the amount of data available.} This provides a great advantage in situations where the form of the vector field is unknown but data is available for learning. {The functional spaces considered are Reproducing Kernel Hilbert Spaces (RKHS) \cite{manton2015primer}, allowing for efficient optimization among other desirable properties.}

A particular difficulty arises when the data is sparse and noisy. This is often the case for longitudinal healthcare data obtained during hospital visits. These visits provide measurements that are sparse in time, with a high level of individual variability. The work presented in this paper has been motivated in part by the need to model the accumulation of the Amyloid protein in the brain of aging subjects.
{Understanding how amyloid contributes to the manifestation of Alzheimer's is a crucial task.  The algorithm discussed here will (we hope) shed more light on the development of this devastating disease.}

Fitting data to nonparametric ODEs is an inverse problem. It requires making assumptions on the initial state of the solution and on the vector field. Furthermore, one needs to make assumptions about the noise model and provide a tractable optimization algorithm. 

We now provide a short bibliographic survey. Further references can be found in the cited papers. First, note that if the time derivative ($\dot x$) was observed, then fitting ODEs to noisy data would reduce to solving a regression problem. This remark has led to the methods known as ``gradient matching'' and to the earliest success in fitting ODEs to data, see e.g. \cite{dondelinger2013ode,brunton2016discovering}. It consists in estimating the gradient from the data, then performing nonparametric regression to fit the vector filed $f$ and eventually, iterating, see \cite{niu2016fast}. These methods become inefficient when the data is sparse and/or noisy.  

Another approach consists in modeling $f$ with polynomials \cite{hu2020revealing}. {Alternatively, one could model $f$ using the units of a Deep Neural Network, see \cite{qin2019data,chen2018neural}.}
These methods integrate the solution along the vector field from guessed initial conditions and compare the resulting trajectories with the observations. Optimization is used iteratively to refine the estimation of $f$ and the initial conditions. Stochastic gradient descent and backpropagation is used in the latter case. 
Another modeling approach is to assume that $f$ belongs to {an}  {RKHS}. This idea, {also known under the name of kernel method,} could be traced back to \cite{koopman1931hamiltonian}. It was successfully applied to fluid mechanics in \cite{schmid2010dynamic}. This is the conceptual approach pursued here. 
We believe that this approach is well-motivated since there is a tight connection between the regularity (smoothness) properties of a kernel and the regularity properties of $f$. Specifically, one can choose an RKHS of vector-valued functions for which one is guaranteed the existence and uniqueness of the corresponding initial value problem. This is a necessary step in proving that more data would result in more accurate predictions. {Another advantage of kernel methods is that there is no need to choose a dictionary of functions as in \cite{brunton2016discovering}. Instead, one selects a kernel, which, our experiments suggest, is easier}. In \cite{dai2022kernel}, the authors assume that each coordinate of the trajectory belongs to a real-valued RKHS where the {functions'} input is   time. In their approach, they first retrieve the full trajectory solving a kernel ridge regression problem. {Next}, they solve for the vector field given the full trajectory {,} assuming that each coordinate of the vector  field can be written as a sum of a linear combination of functions {,which are defined} on each coordinate of the trajectory{.}  {Our framework allows for linear combinations  {of} pairwise products of such functions, as well.} The functions characterizing such a vector field are assumed to be in a {real-valued} RKHS taking a single coordinate as input. {In our approach, we make an assumption on the vector field. {This soft constraint translates to a soft constraint on the set of trajectories}}, without imposing additional constraints on the trajectory itself. As a result, we solve one optimization problem as opposed to the {two-step} approach in \cite{dai2022kernel}. Moreover, {we allow} for higher-order interaction terms compared to the pairwise single coordinates interaction assumed in the mentioned work.
In \cite{heinonen2018learning}, the authors use a Gaussian process (GP) for the vector field. This is the Bayesian counterpart of the frequentist RKHS modeling,  see \cite{kanagawa2018gaussian} for a review of the similarities and differences between RKHSs and GPs. Comparisons between a  collection of algorithms representative of  the state of the art and the proposed algorithm is provided in the experiment section. 


For the purpose of providing a visual and easy to understand illustration of the results generated by the algorithms presented in this paper, please see Figure \ref{fig:Lorenz_dim_plot}. The details of this experiment are provided in  section \ref{sec:oscilator}. We see that the proposed algorithm is able to recover a noisy trajectory and  extrapolate the data, contrary to a method that would use a regression model and ignore the ODE.    
\subsection{Main contributions}

The main contributions of this paper are as follows: 
\begin{enumerate}
    \item We present an RKHS model for fitting nonparametric ODEs to observational data. Conditions for existence and uniqueness of the solutions of the corresponding initial value problem are expressed in terms of the regularity of the kernel; 
    \item We propose a novel algorithm for estimating nonparametric ODEs and the   initial condition(s) from noisy data. This algorithm solves a constrained optimization problem using a penalty method;
   
    
    \item We derive and prove a consistency result for the prediction of the state (interpolation) at unobserved times. This is, up to our knowledge, the first result for the problem of fitting nonparametric ODEs to data.  
    \item We provide experiments with simulated data. We compare the proposed algorithm to 7 existing methods representing state of the art  for various noise levels. We show that our algorithm is competitive.
    \item We provide an experiment modeling the accumulation of Amyloid in the cortex of aging subjects. The data is sparse with, on average, three data points per trajectory (subject) and 179 trajectories. We show competitive performance compared to state of the art. 
\end{enumerate}
\begin{figure}[htb!]
    \centering
    \includegraphics[width=0.95\textwidth]{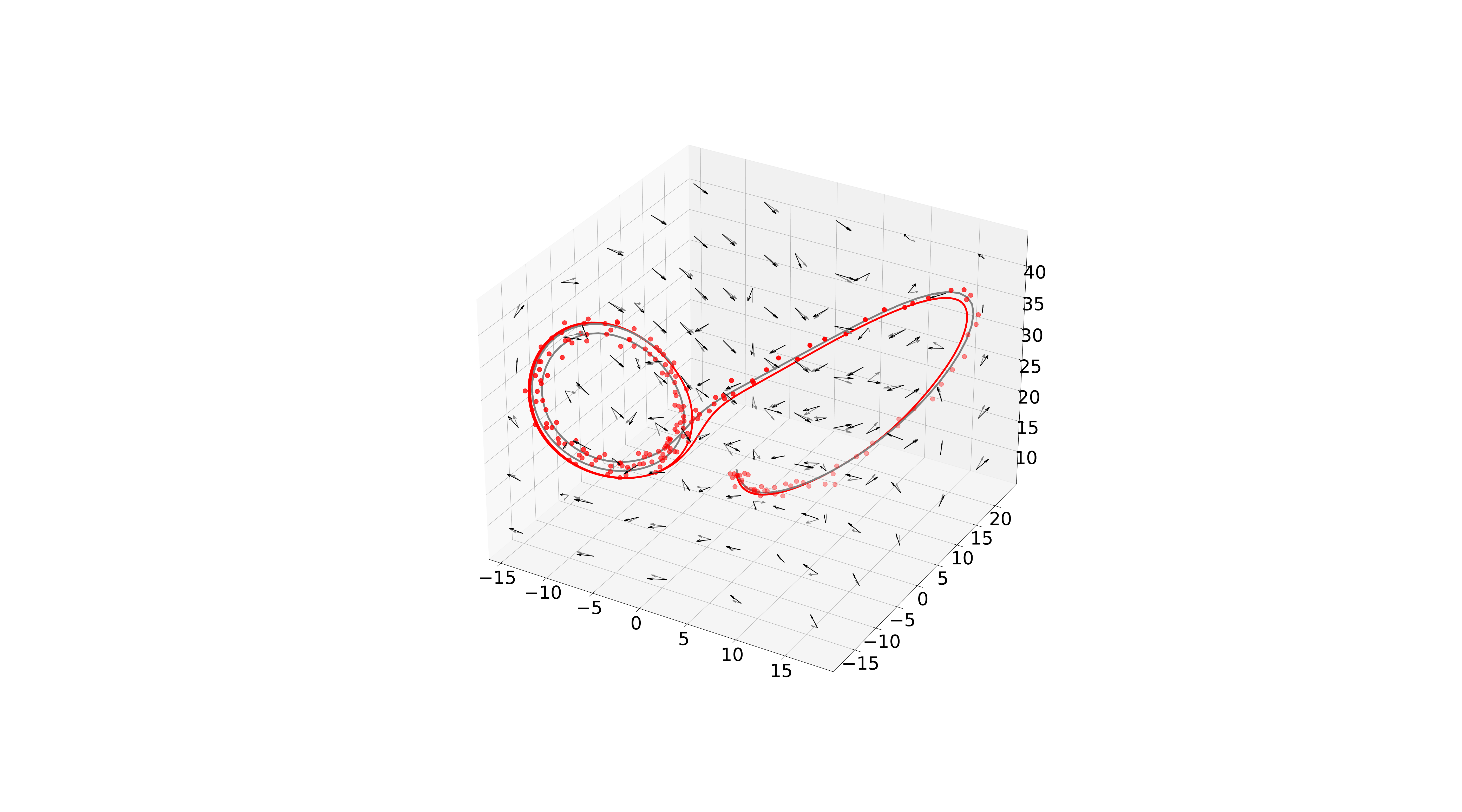}
    \includegraphics[width=.3\textwidth]{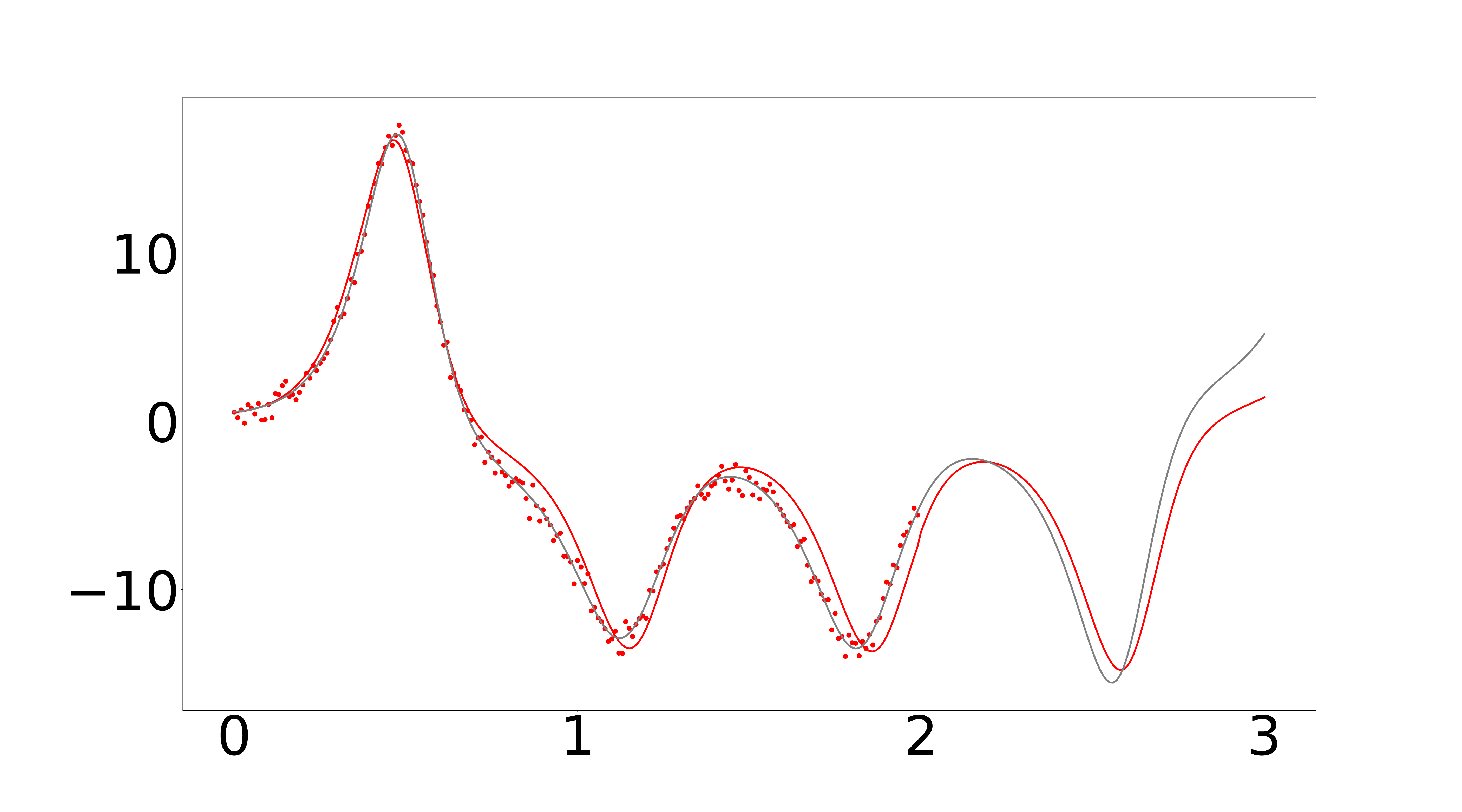}
    \includegraphics[width=.3\textwidth]{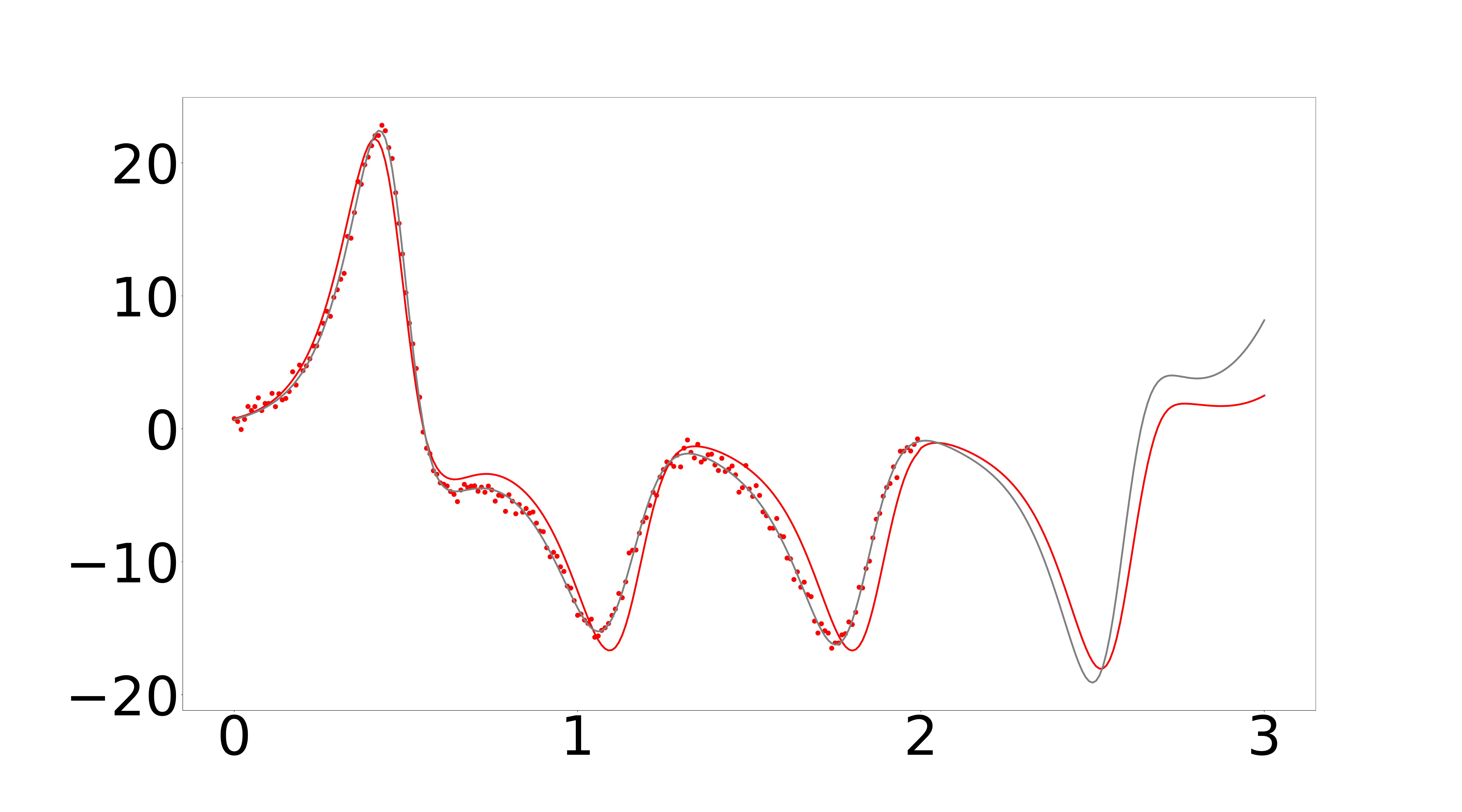}
    \includegraphics[width=.3\textwidth]{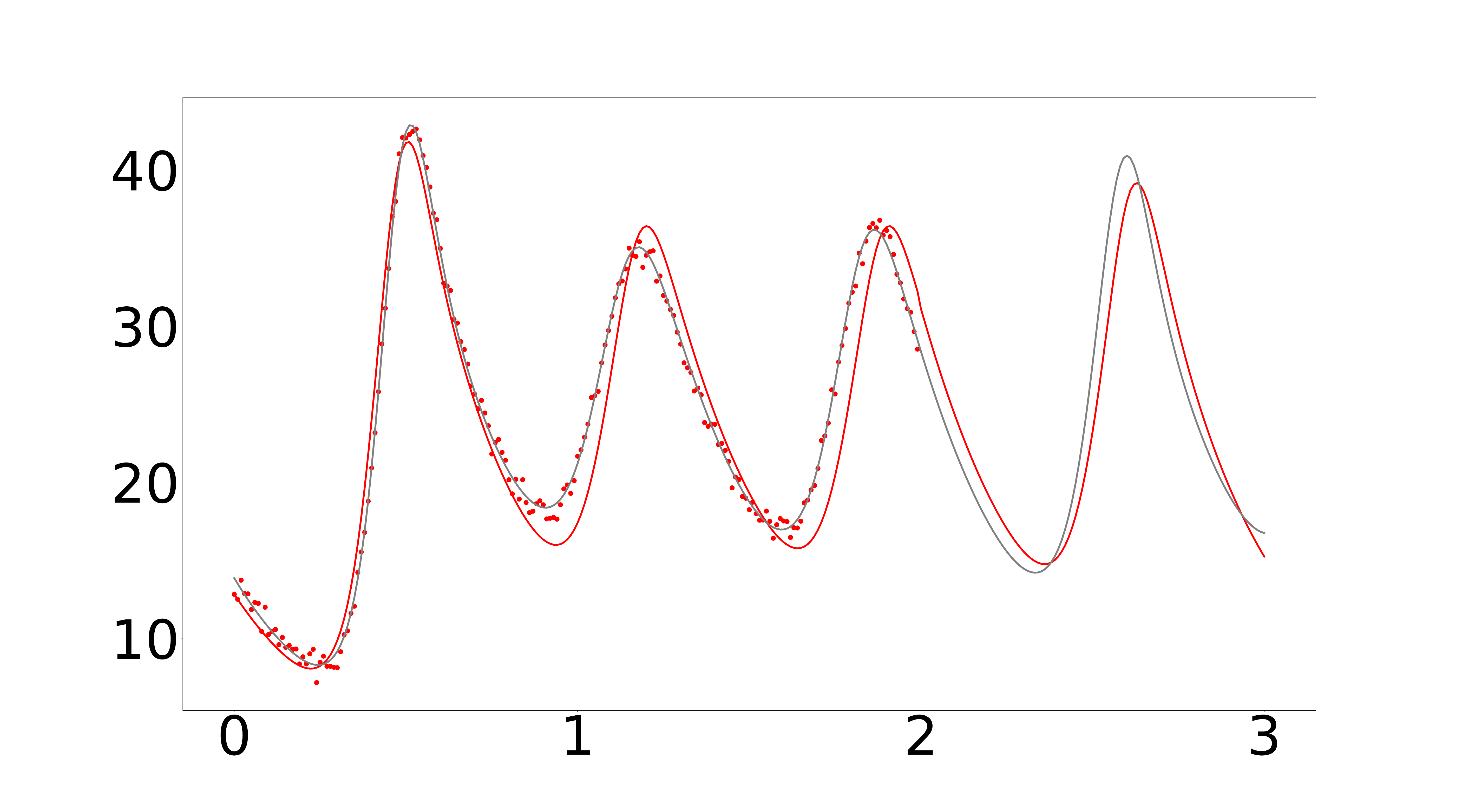}
    \caption{
     (a) Predicted vector field of the Lorenz system.  The Black arrows are the prediction and the grey are the true vector field.  Red points are observations.  The red curve is a predicted trajectory while the grey is the true trajectory.    (b) is the $x$-dimension,  (c) is the $y$-dimension and (d) is the $z$-dimension.  The red points are the observations.  This plot also shows a prediction beyond the last observation in the data. }
    \label{fig:Lorenz_dim_plot}
\end{figure}
The rest of this paper is organized as follows: Section 2 presents {some background material as well as} the model and the algorithms. The consistency results are presented in Section 3 and proved in Appendix A. The experiments appear in Section 4 while Section 5 provides concluding remarks. Appendix B provides examples of kernels.    
\section{Model and algorithm}
{\subsection{Background on Reproducing Kernel Hilbert Spaces (RKHSs)}

Basic notions and notations associated with RKHS are important for understanding the algorithms and derivations presented in this paper. We thus provide a short presentation. We limit ourselves to RKHS over the field of real numbers instead of complex numbers as this is sufficient throughout this paper. We begin with the univariate real-valued case and we continue with the vector-valued case which allows us to describe vector fields, central to this paper.  
\subsubsection{Real-valued RKHS}
Real-valued RKHS are Hilbert spaces of real-valued functions: $\XX \to \RR$, where $\XX$ is a nonempty space. The critical assumption which make them ``reproducing" is that the evaluation functional is continuous. The evaluation functional at $x \in \XX$ is a mapping from a RKHS $H$ to $\RR$, which associates to a function its evaluation at $x$, that is $f \mapsto f(x)$. Thanks to the Riesz representation theorem, evaluating a function in an RKHS is a geometric operation consisting in computing an inner product. Effectively, for any $x \in \XX$, there is a unique vector $k_x \in H$  such that 
\begin{equation}
    f(x)=\inp{f,k_x}_H
\end{equation}
where $\inp{.,.}_H$ is the scalar product associated with $H$. In what follows, we will simply notate $\inp{.,.}$ for this inner product. 
Moreover, let us define, for any $x,y \in \XX$, the so-called kernel 
\begin{equation}
    k(x,y)=\inp{k_x,k_y}
\end{equation}
and let us use this to characterize the function $k_x$. Evaluating $k_x$ at $y$ and using Riesz representation provides 
\begin{equation}
\label{eq:k_x}
k_x(y) = \inp{k_x,k_y}=\inp{k_y,k_x}=k(y,x)    
\end{equation}
Thus the function $k_x(.)$ is the function $k(.,x)$ and for any $f \in H$, 
\begin{equation}
f(x) = \inp{f,k(.,x)}    
\end{equation}
This is the reproducible property of the kernel. Replacing the function $f$ by $k_y$, and using \eqref{eq:k_x}, we obtain that
\begin{equation}
 k_y(x) = \inp{k_y,k(.,x)} =\inp{k(.,y),k(.,x)}=k(y,x) 
\end{equation}
\subsubsection{Vector-valued RKHSs}
Vector-valued RKHSs generalize the real-valued case. The construction is similar. Consider a Hilbert space of functions from $\XX$ to $\RR^d$. Assume, moreover, as in the real-valued case, that the evaluation functional is continuous. Riesz representation theorem then states that for any $x \in \XX$, and $v \in \RR^d$, there exists a unique element in $H$, notated $K_{x,v}$ such that $v^Tf(x) = \inp{f,K_{x,v}}$. The kernel of $H$ is then the  $(d,d)$ matrix where the element $(i,j)$ at the $i^{th}$ row and $j^{th}$ column is defined by 
\begin{equation}
\label{eq:r-prop}
    K_{ij}(x,y)=\inp{K_{x,e_i},K_{y,e_j}}
\end{equation}
where $(e_1,\ldots,e_d)$ is the natural basis of $\RR^d$. Let us use \eqref{eq:r-prop} to characterize the function $K_{x,v}$. We start with $K_{y,e_j}$ and use the reproducing property as well as the symmetry of the inner product. 
\begin{equation}
    e_i^TK_{y,e_j}(x)=\inp{K_{y,e_j},K_{x,e_i}}=\inp{K_{x,e_i},K_{y,e_j}}=K_{ij}(x,y)=e_i^TK(x,y)e_j
\end{equation}
Thus $K_{y,e_j}(.)=K(.,y)e_j$, and
\begin{equation}
\label{eq:repro}
    v^Tf(x)=\inp{f,K_{x,v}}=\inp{f,K(.,x)v} 
\end{equation}
which is the reproducing property for vector-valued RKHS. Applying \eqref{eq:repro} to the function $x \mapsto K(x,y)w$, for $w \in \RR^d$ provides
\begin{equation}
    v^TK(x,y)w = \inp{K(.,y)w,K(.,x)v}=\inp{K(.,x)v,K(.,y)w}
\end{equation}
Lastly, a useful property of the kernel $K$ is that $K(x,y)^T=K(y,x)$. Indeed, 
\begin{multline}
    K_{ji}(x,y)=e_j^TK(x,y)e_i=\inp{K(.,x)e_j,K(.,y)e_i}=\inp{K(.,y)e_i,K(.,x)e_j}=e_i^TK(y,x)e_j=K_{ij}(y,x)
\end{multline}
Choosing $\XX=\RR^d$ allows for defining autonomous vector fields, that is functions $\RR^d \to \RR^d$, and choosing a suitable kernel allows for choosing Lipschitz continuous vector fields as will be discussed in Section \ref{ExUni}.  
}
\subsection{Notations}
{The observations are characterized by multiple time series. There are $n$ times series. The $i^{th}$ one is of length $m_i$. It is characterized by $m_i$ couples $(t_{ij},y_{ij}(t_{ij})),i=1,\ldots,m_i$, where $t_{ij} \in [0,T]$ for some maximum predefined time $T$, and the observations $y_{ij}(t_{ij})$ belong to $\RR^d$.} 

We aim to make predictions at new time points along a {time series} having one or several noisy snapshots. To this end, we explore the following nonparametric ODE model: 
\begin{equation}
\label{eq:model}
\left\{
    \begin{array}{ccc}
       \dot x  & =& f(t,x)\\
        y_{ij}(t_{ij}) & = & x(t_{ij}) + \epsilon_{ij} 
    \end{array}
    \right.
\end{equation}
where $i=1, \ldots, n$, $j=1, \ldots,m_i$. The noise $\epsilon_{ij}$ is bounded or sub-Gaussian. 
This model is nonparametric because $f$ is not specified parametrically. We assume that $f$ belongs to a {RKHS} of smooth functions for which the solution of the ODE exists and is unique, see Section \ref{ExUni}. Background material on RKHS can be found in \cite{hofmann2008kernel} and vector-valued RKHS are reviewed in \cite{alvarez2011kernels}. The rest of the paper is written for the autonomous case when $f(t,x)=f(x)$ {and for the simpler situation where $m_i$ is the same for all time series and when the time points $t_{ij}$ are the same for all the time series i.e. do not depend on $i$.} However, {we will point to the  modifications for the non-autonomous setting when necessary, as well as the situation of non regular sampling.}

\subsection{Existence and uniqueness}
\label{ExUni}
It is a classical result, see \cite{simmons2016differential}, that the initial value problem (IVP):
\begin{equation}
\label{eq:ivp}
    \dot{x}(t)=f(x(t)) \mbox{ and } x(0)=x_0,
\end{equation}
where $f: \RR^d \rightarrow \RR^d$ is Lipschitz continuous has a unique solution defined on the domain $[0, +\infty)$.

Let $H$ be an RKHS of vector-valued functions $\RR^d \mapsto \RR^d$ and let $K$ be the reproducing kernel of $H$. $K$ is a $(d,d)$ matrix-valued kernel. It is then natural to ask: what is a sufficient condition on $K$ which ensures that all $f \in H$ are  Lipschitz continuous? The following  {lemma}  provides an answer. 

\begin{lemma}
\label{cor:LipKernel}
If $f: \RR^d \rightarrow \RR^d$ belongs to an RKHS with kernel $K$ such that:
\begin{multline}
\label{eq:d_K}
d_{K_{ii}}^2(u,v):=K_{ii}\left(u,u\right)-2K_{ii}\left(u,v\right)+K_{ii}\left(v,v\right) \leq  N_K^2 |u-v|^2, \forall u,v \in \RR^{d}, i=1\ldots d,
\end{multline}
for some constant $N_K$, then the IVP problem \eqref{eq:ivp}  has a unique solution defined on  $[0, +\infty)$.
\end{lemma}

\begin{proof}
Notice that for every  {$i=1,\ldots,d$}
\begin{align}
|f_i(u)-f_i(v)|^2 &=|{\left<K(u,\cdot)e_i-K(v,\cdot)e_i,f\right>}_H|^2 \\
                & \leq {||K(u,\cdot)e_i-K(v,\cdot)e_i||}_H^2 {||f||}_H^2 \\
                & = d_{K_{ii}}^2(u,v) {||f||}_H^2
\end{align}
where $e=(e_1,\ldots,e_d)$ is the natural basis of $\RR^d$. Here we have used the reproducing property of the matrix-valued kernel and the Cauchy-Schwartz inequality. 
\end{proof}
Thus, one can choose a kernel that guarantees the existence and uniqueness of the solution of the IVP, which will lead to provable asymptotic performance. We believe that this simple result is a good motivator for the proposed modeling approach.

Let us discuss some examples of kernels satisfying {lemma} 1. The simplest matrix-valued kernels are separable kernels. They are obtained by choosing a scalar kernel $K_1$ and a positive semi-definite matrix $A$. Then, 
\begin{equation}
    K(x,y) = K_1(x,y)A
\end{equation}
The diagonal elements of $K$ are then positive multiples of $K_1$. Thus, if $K_1$ verifies the regularity condition of {lemma 1} , then so do all the separable kernels based on $K_1$. The scalar kernels satisfying the hypothesis of {lemma 1} contain the linear kernel, the Gaussian Kernel, the rational quadratic kernel, the sinc kernel and the mattern kernels for $p>3/2$. Kernels for which the functions in their corresponding RKHSs are not guaranteed to provide unique solutions to the corresponding IVP due to lack of regularity include the polynomial kernels with an order of at least two, the Laplacian kernel and the Mattern kernel for $p\leq 3/2$ Details are provided in  \ref{sec-supp:kernels}. 
{The condition of lemma 1 has a nice interpretation in the case where explicit kernels are used. Indeed, when a feature map associated with the kernel is given explicitly, the conditions of lemma 1 are equivalent to assuming Lipschitz continuous features. The details are provided in the Appendix B. \\} 
{Note on the non-autonomous case: When the vector field is time-dependent denoted by $f(t,x)$, the kernel is defined on $\RR^d \times [0,\infty)$. It is sufficient to assume a global Lipschitz condition with respect to the second variable \cite{simmons2016differential}, namely: There exists a constant $L_K$ such that for every $t\geq 0$ and $u,v \in \RR^d$ and $i \in 1,...,d$:
\begin{equation}
|f_i(t,u)-f_i(t,v)| \leq L_K |u-v|
\end{equation}
It is therefore sufficient to assume a a kernel $K$ defined on $\RR^d \times [0,\infty)$ and satisfying the conditions of lemma \label{cor:LipKernel} as it will ensure the following inequality:
\begin{equation}
d_{K_{ii}}^2(t,u,t,v) \leq N_K^2 |u-v|^2
\end{equation}
}

\subsection{From constrained to unconstrained optimization} 

We first construct the optimization algorithm in the case $n=1$. All the observations are from a single trajectory with the same initial condition. Thus, we temporarily  drop the double indexing with subjects and times to simplify the notation.

Assume the observation times are $t_1<\ldots<t_m$. Consider the following constrained minimization problem: 
\begin{equation}
\label{eq:main}
    \min_{x,f} \frac{1}{m}\sum_{j=1}^m |y_j - x(t_j)|^2 + \lambda ||f-f_0||_H^2,
\end{equation}
under the constraints
\begin{equation}
\label{eq:main_c}
\left\{
\begin{array}{l}
 f  \in  H, \mbox{ the RKHS with matrix-valued kernel }K,\\
x(t) =x(t_1) + \int_{t_1}^t f(x(s))ds, \mbox{ for }t_1\leq t \leq t_m.  
\end{array}
\right.
\end{equation}
The function $f_0 \in H$ is an initial guess for $f$.  Section \ref{sec:ic} describes a gradient matching algorithm for selecting $f_0$. K is a kernel that satisfies  {lemma} 1. 

Consider a regular one-dimensional grid over the interval $[t_1,t_m]$. Specifically, we choose 
\begin{equation}
    s_l=t_1+lh
\end{equation}
with $l=0, \ldots, k$ and we assume that $h$ is small enough so that there are integers $k_1=0 < k_2 < \ldots < k_m$, such that the observation times are 
\begin{equation}
    t_j=t_1+ k_j h,  j=1\ldots m.
\end{equation} 
In practice, the observation times are rounded to fit on this grid. Note that with this notation, $t_j=s_{k_j}$ . We now proceed through a series of transformations to rewrite this constrained optimization problem into an unconstrained one. 

First, we replace the constraints on $x$ by a finite number of constraints as follows:
\begin{equation}
\left\{
\begin{array}{l}
 f  \in  H, \mbox{ the RKHS with kernel }K,\\
x(s_{l+1}) = x(s_l) + \int_{s_l}^{s_{l+1}} f(x(s))ds \\ l=0\ldots k-1. 
\end{array}
\right.
\label{eq:constraints}
\end{equation}
Second, we discretize the constraints using the Euler method of integration:
\begin{equation}
\left\{
\begin{array}{l}
 f  \in  H, \mbox{ the RKHS with kernel }K,\\
x(s_{l+1}) = x(s_l) + hf(x(s_l)) \\\mbox{ for } l=0\ldots k-1. 
\end{array}
\right.
\label{eq:Euler}
\end{equation}
Third, we replace the constrained optimization problem by an unconstrained one using a single Lagrange constant $\gamma>0$. Notate $z_l=x(s_l)$, $l=0\ldots k$, 
\begin{equation}
\min_{z\in \RR^{d(k+1)},f\in H} J(z,f,\gamma),
\label{eq:unconstrained}
\end{equation}
with
\begin{equation}
    J(z,f,\gamma)=\frac{1}{m}\sum_{j=1}^m |y_j - z_{k_j}|^2+ \gamma \frac{1}{k} \sum_{l=0}^{k-1}\left|z_{l+1}-z_l-hf(z_l)\right|^2  
     + \lambda ||f-f_0||_H^2 . 
    \label{eq:JEuler}
\end{equation}
{It is instructive to remark the similarities between the loss function in equation \ref{eq:JEuler}  and the loss proposed  in Physics-informed Neural Networks \cite{RAISSI2019686}, where the observations are generated from an unknown partial differential equation. Indeed, the total loss function in Physics-informed Neural Networks can be decomposed as a sum of two functions: One that measures the deviation of solution from the observations, and the second usually defined as the residual function term, measures the violation of the partial differential equation constraint that the solution must satisfy. In our context,
$$\frac{1}{m}\sum_{j=1}^m |y_j - z_{k_j}|^2$$ corresponds to first function, and $$\frac{1}{k} \sum_{l=0}^{k-1}\left|z_{l+1}-z_l-hf(z_l)\right|^2$$ corresponds to the residual function term.
However there are some notable differences. In physics informed neural networks, the form of the PDE is known up to finite dimensional parameters. The loss is viewed as a function of the solution to the partial differential equation and these finite dimensional parameters. The solution itself is modeled by a neural network. In our case, the loss is viewed as a function of the vector field and the initial solution. The differential equation is therefore characterized by the RKHS, usually infinite-dimensional. Moreover, equation \ref{eq:JEuler} contains a regularization term penalizing vector fields with large RKHS norm, which is typical of loss function parametrized by RKHS functions.}
\subsection{Penalty method}
The penalty method is an iterative method that consists of enforcing the constraints by increasing a penalty parameter, in this case $\gamma$. The schematic of the method is presented in Algorithm \ref{alg:ODE-RKHS}. At each step, the functional $J(z,f,\gamma)$ in \eqref{eq:JEuler} is minimized with respect to $(z,f)$, for a fixed value of $\gamma$. Then,  $\gamma$ is increased. The optimization for $(z,f)$ is done asynchronously, first optimizing over $z$ for a fixed $f$, then optimizing over $f$ for the newly updated $z$.  

Let us now describe these optimization steps in more detail. For a fixed $\gamma$ and $f$, $J(z,f,\gamma)$ in \eqref{eq:JEuler} is non-convex in $z$ due to the presence of $f(z_l)$. Therefore we replace $f$ by its first-order Taylor expansion evaluated at the value $z_l^{(s)}$ obtained in the previous iteration $s$: 
\begin{equation}
\label{eq:Taylor}
    f(z_l) \approx f(z_l^{(s)}) + (z_l-z_l^{(s)})^T\nabla_{z_l}f(z_l^{(s)})
\end{equation}
Note that with this approximation, $J$ is convex, quadratic, and sparse in $z$. This allows the use of an efficient linear solver for this minimization. The number of unknowns is $d(k+1)$. \\
{Note on the non-autonomous case: When the vector field 
is time-dependent, the vector field is evaluated at points of the form $f(t_l,z_l)$. Notice that the $t_l$'s are the time points of the grid, therefore fixed and known. Hence, the linearization in equation \eqref{eq:Taylor} is made only with respect to the space variable:
\begin{equation}
\label{eq:Taylor}
    f(z_l,t_l) \approx f(z_l^{(s)},t_l) + (z_l-z_l^{(s)})^T\nabla_{z_l}f(z_l^{(s)},t_l)
\end{equation}
}
For a fixed $\gamma$ and $z$, minimizing $J$ in $f$ is equivalent to a multivariate kernel ridge regression problem. After the change of variable, $g=f-f_0$, and setting
\begin{equation}
    u_l = (z_{l+1}-z_l)/h - f_0(z_l), l=0\ldots k-1, 
\end{equation}
we use the representer theorem to show that the minimizer in $f \in H$ of $J$ is of the form
\begin{equation}
    f(z) = f_0(z)+\sum_{l=0}^k K(z,z_l)w_l,
\end{equation}
where $w_l \in \RR^d$. 
Let $W=(w_1^T,\ldots,w_{k+1}^T)$, be of dimension $(d(k+1),1)$ and similarly let $U=(u_1^T,\ldots,u_{k+1}^T)$ and $K$ be the matrix with $(d,d)$ block element $K_{kl}=K(x_k,x_l)$. We find that $W$ is a minimizer of the convex quadratic function 
\begin{equation}
    \frac{\gamma h^2}{k}|U-KW|^2 + \lambda W^TKW
\end{equation}
and thus $W$ is the solution to the linear system: 
\begin{equation}
    \left(K + \frac{\lambda k}{\gamma h^2}I\right) W=U
\end{equation}
The schematic algorithm is provided in Algorithm \ref{alg:ODE-RKHS}. 
\begin{algorithm}[tb]
\begin{algorithmic}[1]
\STATE Init: $h,\rho,\lambda,f^{(0)},\gamma^{(0)},s=0$
 \WHILE{termination condition is not met}
  \STATE $z^{(s+1)} \leftarrow \arg\min_{z \in \RR^{d(k+1)}} J(z,f^{(s)},\gamma^{(s)})$ 
  \STATE $f^{(s+1)} \leftarrow \arg \min_{f \in H} J(z^{(s+1)},f,\gamma^{(s)})$
  \STATE $\gamma^{(s+1)} \leftarrow \gamma^{(s)}(1+\rho)$
  \STATE $s=s+1$
  \STATE Check termination condition\;
 \ENDWHILE
\end{algorithmic}
 \caption{Penalty method for ODE-RKHS}
 \label{alg:ODE-RKHS}
\end{algorithm}

\subsection{Initial condition and termination criteria}
\label{sec:ic}
Since the algorithm will converge to a local minimum of the cost function, the choice of the initial condition is important. We use a gradient-matching method. 
\begin{enumerate}
    \item Approximate the time derivatives of $x$ at the observed times $\dot x(t_j)$, denoted $\hat{\dot {x}}(t_j)$
    \item Estimate $f_0 \in H$ using ridge regression, i.e. minimize over $H$
    \begin{equation}
        G(f_0) = \frac{1}{m} \sum_{j=1}^m |\hat{\dot{x}}(t_j) - f_0(y_j)|^2 + \lambda ||f_0||_H^2
        \label{eq:GradientMatching}
    \end{equation}
\end{enumerate}  
There are several possibilities for the approximation in the first step depending on the sparsity of the data and the amount of noise. In the experiments below, we use central differences. 

The termination condition of Algorithm \ref{alg:ODE-RKHS} includes a fixed number of iterations $S$ and a threshold on the quantity  $||f^{(s+1)}-f^{(s)}||/||f^{(s)}||$ which allows for early stopping.  

\subsection{Multiple trajectories}
We present here the extension of the method to multiple trajectories, say $n>1$ subjects. We assume the same number of observations for each subject and regular sampling to simplify the presentation. 

First, we replace \eqref{eq:JEuler} and \eqref{eq:constraints} with 
\begin{equation}
    \min_{x,f} \frac{1}{nm}\sum_{i=1}^n\sum_{j=1}^m |y_{ij} - x_i(t_{ij})|^2 + \lambda ||f-f_0||_H^2,
\end{equation}
under the constraints
\begin{equation}
\left\{
\begin{array}{l}
 f  \in  H, \mbox{ the RKHS with matrix-valued kernel K},\\
x_i(t) =x_i(t_1) + \int_{t_1}^t f(x_i(s))ds,\\ \mbox{ for }t_1\leq t \leq t_m, i=1 \ldots n
\end{array}
\right.
\end{equation}
We then proceed along the same steps as for the single trajectory case, leading to the unconstrained optimization problem, generalizing \eqref{eq:unconstrained} and \eqref{eq:JEuler}. 

Notate $z_{il}=x_i(s_l)$, $l=0\ldots k$, $i=1 \ldots n$, and $z=(z_1,\ldots,z_n)$  
\begin{equation}
\label{eq:unconstained}
\min_{z\in \RR^{nd(k+1)},f\in H} J_{\text{multi}}(z,f,\gamma),
\end{equation}
with
\begin{multline}
    J_{\text{multi}}(z,f,\gamma)=\frac{1}{nm}\sum_{i=1}^n\sum_{j=1}^m |y_{ij} - z_{ik_j}|^2  + \gamma \frac{1}{nk} \sum_{i=1}^n \sum_{l=0}^{k-1}\left|z_{i,l+1}-z_{il}-hf(z_{il})\right|^2 + \lambda ||f-f_0||_H^2 . 
    \label{eq:JEuler Multi}
\end{multline}
The key point is that $J_{\text{multi}}$ decouples the trajectories such that the optimization over $z$ can be carried out separately for each trajectory. However, all the observations contribute to the estimation of $f$.  The algorithm is presented in Alg \ref{alg:ODE-RKHS Multi}. In Line 6: we use the no-trick formulation using Gaussian quadrature Fourier features as described in \cite{dao2017gaussian}.
\begin{algorithm}[tb]
\begin{algorithmic}[1]
\STATE Init: $h,\rho,\lambda,f^{(0)},\gamma^{(0)},s=0$
 \WHILE{termination condition is not met}
 \FOR{$i=1\ldots n$}
  \STATE $z_i^{(s+1)} \leftarrow \arg\min_{z_i \in \RR^{d(k+1)}} J_{\text{multi}}(z,f^{(s)},\gamma^{(s)})$
  \ENDFOR
  \STATE $f^{(s+1)} \leftarrow \arg \min_{f \in H} J_{\text{multi}}(z^{(s+1)},f,\gamma^{(s)})$
  \STATE $\gamma^{(s+1)} \leftarrow \gamma^{(s)}(1+\rho)$
  \STATE $s=s+1$
  \STATE Check termination condition\;
 \ENDWHILE
\end{algorithmic}
 \caption{Multi Trajectories Penalty method for ODE-RKHS}
 \label{alg:ODE-RKHS Multi}
\end{algorithm}

\subsection{Computational Complexity}
We analyze the complexity of the algorithm Alg \ref{alg:ODE-RKHS Multi}. The key parameters are: 
\begin{enumerate}
    \item $d$: the dimension of the observed vectors;
    \item $n$: the number of observed trajectories;
    \item $k$: the number of samples in the discretization of the time interval;
    \item $S$: the number of steps in Alg \ref{alg:ODE-RKHS Multi};
    \item $n_F$: the number of Fourier features. 
\end{enumerate}
We use $O(p^3)$ for the time complexity of solving a (dense) linear system with $p$ variables and $O(w^2p)$ in the case of a band matrix of width $w$, see \cite{kilicc2013inverse}. Alg \ref{alg:ODE-RKHS Multi}, line 4 consists in solving a linear system of size $dk$ with a band matrix of bandwidth $w=3d$, thus   $O(kd^3)$ computations. Line 6 consists in solving $d$ full linear systems of dimension $n_F$, thus $O(dn_F^3)$ computations. 
In total, we find $O(Snkd^3 + Sdn_F^3)$. Note that $k$ is typically chosen proportional to the average number of data points per trajectory. Thus, overall, the algorithm is linear in the number of observations but cubic in the dimension of the observations.

\subsection{Non autonomous systems, covariates, and irregular sampling}
Non autonomous systems and covariates are handled by modifying the kernel. The issue of irregular sampling is addressed by replacing the first term of \eqref{eq:JEuler}  by  
\begin{equation}
    \frac{1}{n} \sum_{i=1}^n \sum_{j=1}^{m_i} (t_{i,j+1}-t_{ij})|y_{ij} - z_{ik_j}|^2
\end{equation}
with $t_{i,m_i+1}=T$, $i=1\ldots,n$

\section{Consistency of the solution: A finite sample result}
\label{section:cons}
In this section, we assume that our algorithm solves the following optimization problem (where $t_{m+1}=T$ by definition):
\begin{equation}
\min_{\RR^{d(k+1)},f \in H}\sum_{j=1}^m (t_{j+1}-t_j)|y_j - z_{k_j}|^2,
\end{equation}
Under the constraints:
\begin{enumerate}
    \item ${||f-f_0||}_H \leq R$, $|z_0| \leq r$
    \item $z_{l+1}=z_l+hf(z_l)$, $0\leq l \leq k$
\end{enumerate}

Notice that constraint  2 corresponds to the  Euler method for the ODE: $\dot{x}=f(x)$. Therefore, by linearly interpolating between the times of subdivision $s_l$, $0\leq l \leq k$, we can generate a solution $\hat{x}(\cdot)$ defined on $[0,T]$. We denote by $x^{*}(\cdot)$ the true trajectory generating the noisy observations $y_j$ at each time $t_j$. The purpose of this section is to present a result controlling (in probability) the $L^2$ norm squared of $\hat{x}-x^{*}$ :
\begin{equation}
{||\hat{x}-x^{*}||}_{L^2}^2:=\int_{0}^T{|(\hat{x}(t)-x^{*}(t))|}^2 dt
\end{equation}
Let us make the following assumptions:
\begin{itemize}
\item $\mathbf{A_1}$: There exists an $f^* \in H, {||f^*-f_0||}_H \leq R$ and ${|x_0^*|} \leq r$ such that $x^*(0)=x_0^*$ and $\dot{x}^*(t)=f^*(x^*(t))$ for every $0 \leq t \leq T$.
\item  $\mathbf{A_2}$: The noise variables $\epsilon_{ij}$ are independent and  bounded in absolute value by a constant $M_{\epsilon}$. (We can assume that the variables are subgaussian instead of bounded if we want to generalize this result)
\item  $\mathbf{A_3}$: The kernel $K$ is $\mathcal{C}^2(\RR^d)$ in its first argument (this implies that it is also $\mathcal{C}^2(\RR^d)$ in its second argument).
\item $\mathbf{A_4}$: The kernel $K$ satisfies \eqref{eq:d_K}. 

\end{itemize}

We refer to section \ref{ExUni} for examples of kernels satisfying $\mathbf{A_3}$ and $\mathbf{A_4}$.

These assumptions are sufficient for obtaining the  main theorem of this section, controlling ${||\hat{x}-x^{*}||}_{L^2}^2$ with high probability. 
\begin{theorem}
\label{MainTheo3}
Assuming $\mathbf{A_1}, \mathbf{A_2}$, $\mathbf{A_3}$ and  $\mathbf{A_4}$,  there exist  positive constants \\ $K_1$,$K_{2}$,$K_{3}$  and $K_{4}$, depending only on $R$, $r$, $T$, $M_\epsilon$, $N_K$ and the kernel $K$ 
such that for every $\epsilon>0$, with probability less than $\exp{\left(\frac{-K_{2}\epsilon^2}{d\sum_{j=1}^m{(t_{j+1}-t_j)}^2}\right)}:$
\begin{equation}
\label{eq:MainIneq}
{||\hat{x}-x^{*}||}_{L_2}^2   \geq  K_1d \sqrt{\sum_{j=1}^{m}{(t_{j+1}-t_{j})}^2} + h^2K_{3}d+K_{4}d\sum_{j=1}^{m}{(t_{j+1}-t_{j})}^2+\epsilon. 
\end{equation}

\end{theorem}

For a better understanding of Theorem \ref{MainTheo3}, assume a regular sampling of the interval $[0,T]$ with $m$ points, so that for every $j$, $t_{j+1}-t_j=\frac{1}{m}$. In that case, under the same hypothesis, for any $\epsilon>0$, with probability less than  $\exp{\left(\frac{-K_{2}m\epsilon^2}{d}\right)}$:

\begin{equation}
{||\hat{x}-x^{*}||}_{L_2}^2   \geq     \frac{K_1d}{\sqrt{m}} +\frac{K_{4}d}{m}+h^2K_{3}d+\epsilon. 
\end{equation}

A proof of Theorem \ref{MainTheo3} is provided in the appendix. We provide here a description of the main ideas. The third term in the right hand side of inequality \eqref{eq:MainIneq} corresponds to the  global truncation error between the numerical solution of the ODE and the true solution. The second term corresponds to the error between ${||\hat{x}-x^{*}||}_{L^2}^2$ and  $\frac{1}{m}\sum_{j=1}^m |x^*(t_j) - \hat{x}(t_j)|^2$. The first  term is the leading term, assuming that $h$ is always less than $\frac{1}{m}$. Assume that $\hat{x}$ solves the continuous-constraints optimization problem (without an Euler approximation), i.e:

\begin{equation}
\min_{x,f}\frac{1}{m}\sum_{j=1}^m|y_j - x(t_j)|^2,
\end{equation}
Under the constraints: ${||f-f_0||}_H \leq R$, $|x_0| \leq r$ and $x(t)=x_0+\int_{0}^{t}f(x(u))du,$ $\forall 0 \leq t \leq T$, we can then consider the ``generalization" error:
\begin{equation}
\frac{1}{m}\sum_{j=1}^m|x^*(t_j) - \hat{x}(t_j)|^2.
\end{equation}
An upper bound of this error is given by the first term. The main tool used to obtain the upper bound is Dudley's chaining inequality, see \cite{vershynin2018high}. We notice that for every $i=1,\dots,d$, the set of coordinate functions $x_i$, where $x$ and $f$ satisfy the constraints of the continuous problem, is included in a set of functions that are uniformly Lipschitz continuous and bounded (the Lipschitz constant and bound does not depend on $x_0$ and $f$). Upper bounds of covering numbers of such functions are well-known, see \cite{vershynin2018high}, hence the use of Dudley's inequality. \\
{One can easily transform the inequality on the probability of theorem \ref{MainTheo3} to an  inequality on $\mathbb{E}\left({||\hat{x}-x^{*}||}_{L_2}^2\right)$. Indeed, let us assume for simplicity a regular sampling of $m$ points the interval $[0,T]$. We denote by:
\begin{equation}
\hat{E}_{L_2}:={||\hat{x}-x^{*}||}_{L_2}^2-\frac{K_1d}{\sqrt{m}} -\frac{K_{4}d}{m}-h^2K_{3}d.
\end{equation}
Using theorem \ref{MainTheo3}, we have the following inequality:
\begin{align}
\mathbb{E}\left(|\hat{E}_{L_2}|\right) & = \int_{0}^{\infty}\mathbb{P}\left(|\hat{E}_{L_2}|\geq \epsilon\right) \\
                                       & \leq \int_{0}^{\infty}\exp{\left(\frac{-K_{2}m\epsilon^2}{d}\right)} \\
                                       &=\sqrt{\frac{{\pi}}{{4K_2}}}\sqrt{\frac{d}{m}}
\end{align}

This implies the following result.
\begin{cor}
Assume we have a regular sampling of $m$ points on the interval $[0,T]$. Then:
\begin{equation}
\label{eq:ExpectedTheo}
\mathbb{E}\left({||\hat{x}-x^{*}||}_{L_2}^2\right) \leq \sqrt{\frac{{\pi}}{{4K_2}}}\sqrt{\frac{d}{m}}+\frac{K_1d}{\sqrt{m}} +\frac{K_{4}d}{m}+h^2K_{3}d.
\end{equation}
\end{cor}

}
{ To illustrate the inequality in \eqref{eq:ExpectedTheo}, we conducted a simple toy experiment where the conditions of the theorem are satisfied, and evaluated the convergence rate. In this experiment we considered a one-dimensional autonomous system. We randomly initialized the weights of a function determined by 200 Fourier random features, recorded the norm of the function, and generated a trajectory of 5120 samples using this function. Then we took ten independent and identically distributed random samples of noise with a standard deviation of .05. This provided us with 10 noisy trajectories of 5120 samples (of the same trajectory but different samples of noise). Finally, we sub-sampled each of these ten noisy trajectories to get 2560 samples, 1280 samples, ... all the way down to 5 samples. This gave us 10 training sets, each with 5, 10, 20, 40, ..., 5120 samples. We trained the algorithm on each of these datasets and reported the average $L_2$ (squared) error between the estimated trajectory and the true one over the ten trajectories at each level of sparsity. In figure \ref{fig:Convergence}, we provide a plot of the log of the average $L_2$ (squared) errors as a function of the log of the number of samples used during training. Equation \eqref {eq:ExpectedTheo} predicts a slope at least $-.5$. We fit the data to a line of slope $-.8$, consistent with \eqref {eq:ExpectedTheo}. We provide a plot with a line of slope $-.5$ for comparison.}
\begin{figure}[htbp]
\label{fig:Convergence}
\includegraphics[width=.5\textwidth]{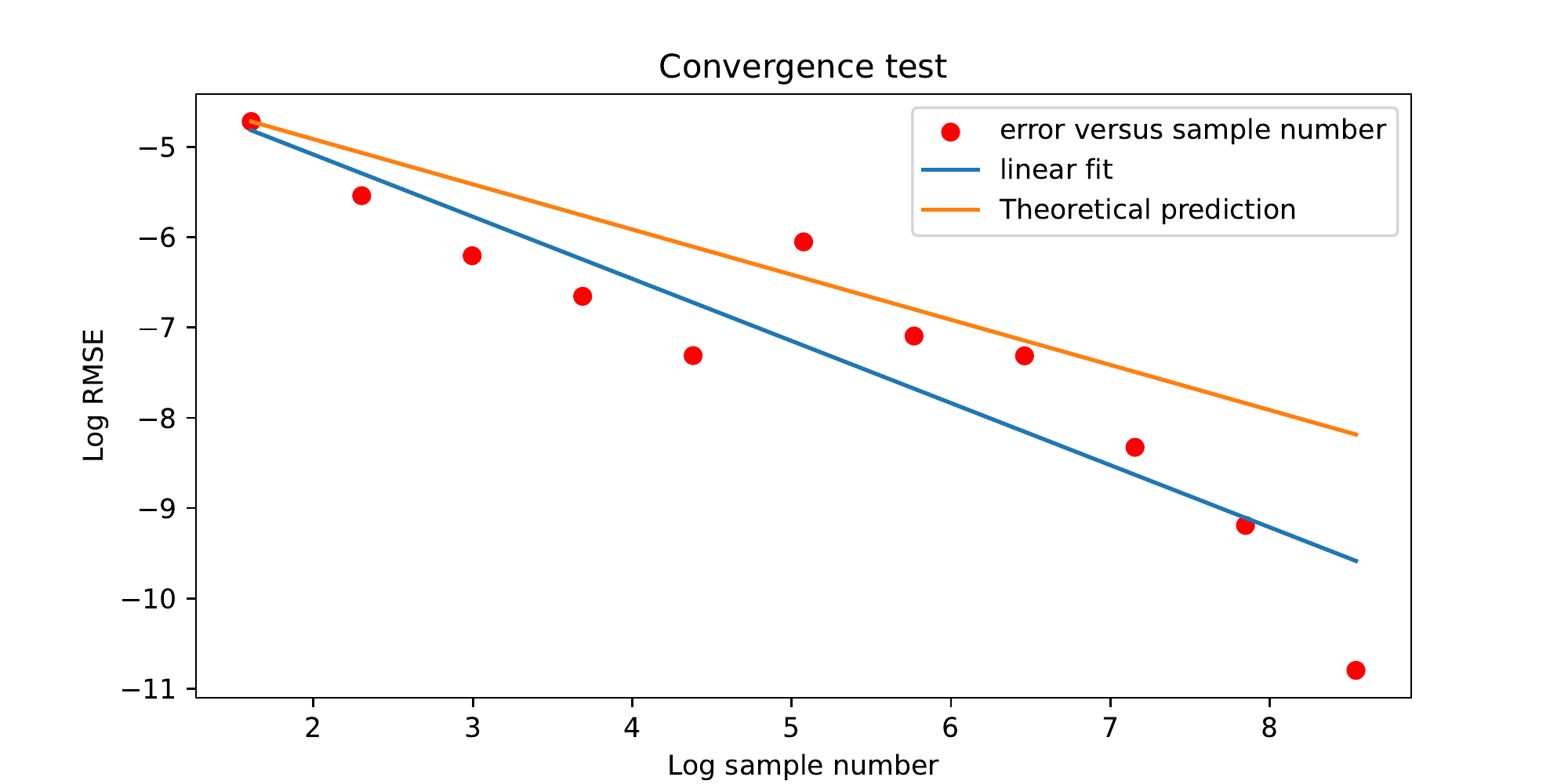}
\includegraphics[width=.5\textwidth]{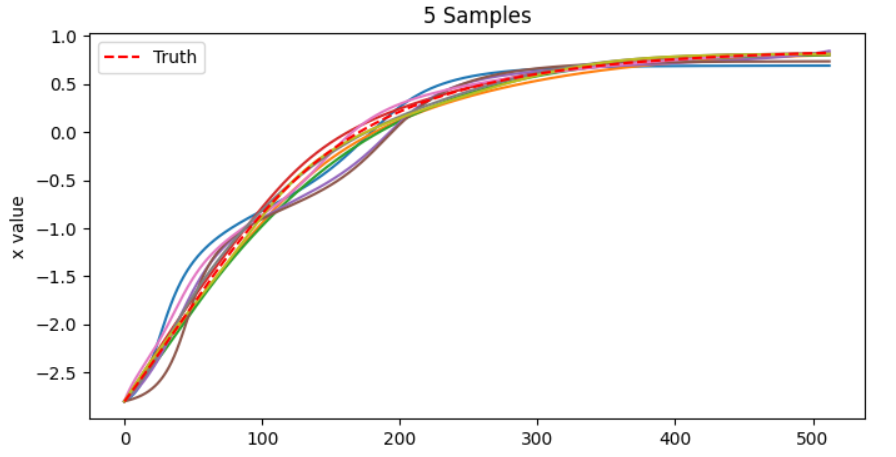}
\caption{On the left we plot the log of the average $L_2$ squared error between the true trajectory and the estimated one as a function of the log of the number of samples. A linear regression yields a slope of $-.8$ indicating convergence at a rate between $\frac{1}{\sqrt{m}}$ and  $\frac{1}{m}$. On the right we plot the predicted trajectories when we use have 5 observations, together with the true trajectory (in the dotted line).}
\end{figure}

\section{Experiments}
We report experiments for simulated data as well as for real data. In each case, we compare the performances of the proposed algorithm, generically named ODE-RKHS, with seven other algorithms. These algorithms constitute, up to our knowledge,  the current state of the art  for learning nonparametric ODEs from noisy data. We briefly review these algorithms and provide references below.  
\begin{enumerate}

\item \emph{Nonparametric Ordinary Differential Equations:}
Nonparametric Ordinary Differential Equations (npODE) is presented in \cite{heinonen2018learning}. The authors use a Bayesian model with Gaussian processes (GP). It is the Bayesian counterpart of the frequentist model presented in this paper. Unlike GP regression where the optimization can be computed in closed form, an approximate optimization method is required. The authors use inducing points, see \cite{quinonero2005unifying} and sensitivity equations, see \cite{kokotovic1967direct}. The npODE code was downloaded from \url{http://www.github.com/cagatayyildiz/npode} in February 2021. {Given the normalized trajectory sets, we ran the algorithm with a scale factor of 1 and an $\ell_0$ of 1. For the 2D systems, we used a width of the inducing point grid $W=6$, matching the demonstration examples. For the 6D Lorenz96, we encountered out-of-memory errors for $W>2$, possibly indicating an empirical scaling issue with the method. We thus used $W=2$ for this system.}

\item \emph{Sparse Identification of Nonlinear Dynamics (Fourier and Polynomial Candidate Functions):}
Sparse Identification of Nonlinear Dynamics (SINDy) is a highly cited technique for identifying nonlinear dynamics from data, see\cite{brunton2016discovering}. SINDy predicts governing dynamics equations using gradient matching via sparse regression.  In the experiments shown, we test SINDy with two different libraries of possible functions: polynomials up to order three and Fourier features. We choose the SR3 sparsity regularization for its superior performance, detailed in \cite{zheng2018unified},
 which has a threshold value as a hyperparameter. Other hyperparameters in our tests include the polynomial library's degree and the size and lengthscale of the Fourier features library. A grid search tuner was employed to determine the best hyperparameter values, with the same holdout and evaluation sets as in the competing algorithms. pySINDy v1.6.3 was used for the implementation \cite{desilva2020}. We use the AutoKoopman library to tune the hyperparameters, described in \cite{Lew2023}.

\item \emph{Extended Dynamic Mode Decomposition:}
The Koopman operator is an infinite dimensional linear operator that captures the dynamics of
a non-linear dynamical system.
Dynamic Mode Decomposition (DMD), described in \cite{schmid2010dynamic}, can approximate the Koopman operator's eigenvalues and eigenvectors based on observations of the system state. Extended DMD (EDMD) generalizes to nonlinear systems learning
 by approximating the Koopman operator in a high-dimensional space of observables, see \cite{williams2015data}. These observables must be selected before using EDMD, and can be chosen ad-hoc or by using library learning methods \cite{yeung2019learning}. We use random Fourier features as the observable functions for these experiments, specified in \cite{degennaro2019scalable}. We use the AutoKoopman library to tune the hyperparameters via Bayesian optimization{, available at \url{https://github.com/EthanJamesLew/AutoKoopman}.}

\item  \emph{Kernel Analog Forecasting:}
Analog forecasting is a time series prediction method that utilizes the idea of analog forecasting that follows the evolution of a historical time series that most closely matches the current state.  Kernel analog forecasting (KAF) replaces single-analog forecasting with weighted ensembles of analogs constructed using local similarity kernels that employ several dynamics-dependent features designed to improve forecast skill \cite{zhao2016analog} \cite{burov2021kernel}.  Our KAF implementation {is based on  \url{https://github.com/rward314/StreamingKAF}.} Hyperparameters are the kernel function and rank used for the number of eigenvalues found from the data-defined kernel matrix. We selected a Gaussian kernel and grid tuned for rank and kernel lengthscale. {We use the same eigenvalue multiplier of $10^{-4}$ as the referenced code.}

\item  \emph{Sparse Cyclic Recovery:}
We implement the method formulated in \cite{schaeffer2020extracting} well-suited for the experiments as it is designed for learning structured dynamical systems from under-sampled and possibly noisy state-space measurements. For index invariant systems, the method generates cyclic permutations to augment the training data. Then, it builds a library of Legendre polynomials of candidate functions and does basis pursuit with thresholding to recover the dynamics. The hyper-parameters involved are the parameters for the Douglas-Rachford algorithm used to solve the Legendre basis pursuit (L-BP) problem and the Legendre polynomial degree;  we tune these parameters via grid search. {We referenced the parameters used in their GitHub project \url{https://github.com/linanzhang/SparseCyclicRecovery}. We utilize the same candidate functions as the paper, but tune the noise threshold $\sigma$ and the $\mu$, $\tau$ parameters of the optimizer. Because of compute effort limitations, we set the maximum number of optimization iterations to $10^4$.}

\item \emph{Gradient descent via optimal control:}
We implemented a gradient descent algorithm based on the co-state equations derived from optimal control theory, see \cite{pontryagin1987mathematical}. Specifically, we compute the gradient of the likelihood function under the constraints provided by the Euler discretization of the ODE. This optimization under equality constraints is performed using a co-state as explained in \cite{JMLR:v21:18-415}. The algorithm effectively implements a backpropagation algorithm in a deep neural network with parameters shared among all layers, see also \cite{chen2018neural}. 

\end{enumerate}

The Amyloid data is presented in section \ref{sec:Amyloid}. This dataset has motivated the creation of the ODE-RKHS algorithm. A smooth vector field, many trajectories, and few sparse and noisy observations per trajectory characterize it. 

\subsection{Selection of the hyper-parameters for the  ODE-RKHS algorithm}
We first select the parameter $h$, the time discretization. A smaller $h$ provides better accuracy at the cost of a linear increase in computational time. Next, we select the parameter $\gamma^{(0)}$ small enough such that the data term in \eqref{eq:JEuler Multi} would be the dominant term. Finally, we performed a grid search for the parameters $\lambda$ and $\rho$, using a validation set consisting of 20\% of the available data in each case. 

\subsection{Oscillator data} 
\label{sec:oscilator}
The FitzHugh-Nagumo (FHN) oscillator data is a controlled experiment with known and easy-to-visualize 2D trajectories. It has helped calibrate the algorithm described in this paper. It was also demonstrated in \cite{heinonen2018learning} for the npODE algorithm. We ran experiments using a simulated dataset generated as follows:
\begin{equation}
\begin{split}
\dot{v} &= v - v^3/3 - w + 1\\
\dot{w} &= 0.08(v+0.7-0.8w)
\end{split}
\end{equation}
Intermediate and final results of the ODE-RKHS algorithm are presented in Fig. \ref{fig:run} for the FHN data. Notice that during the first steps, shown on the top line, the estimated trajectories with solid color lines are rough but fit the data closely. During the later steps, shown on the bottom line, the trajectories are smoother but still fit the data.  
\begin{figure}[htb!]
    \centering
    \includegraphics[width=0.48\textwidth]{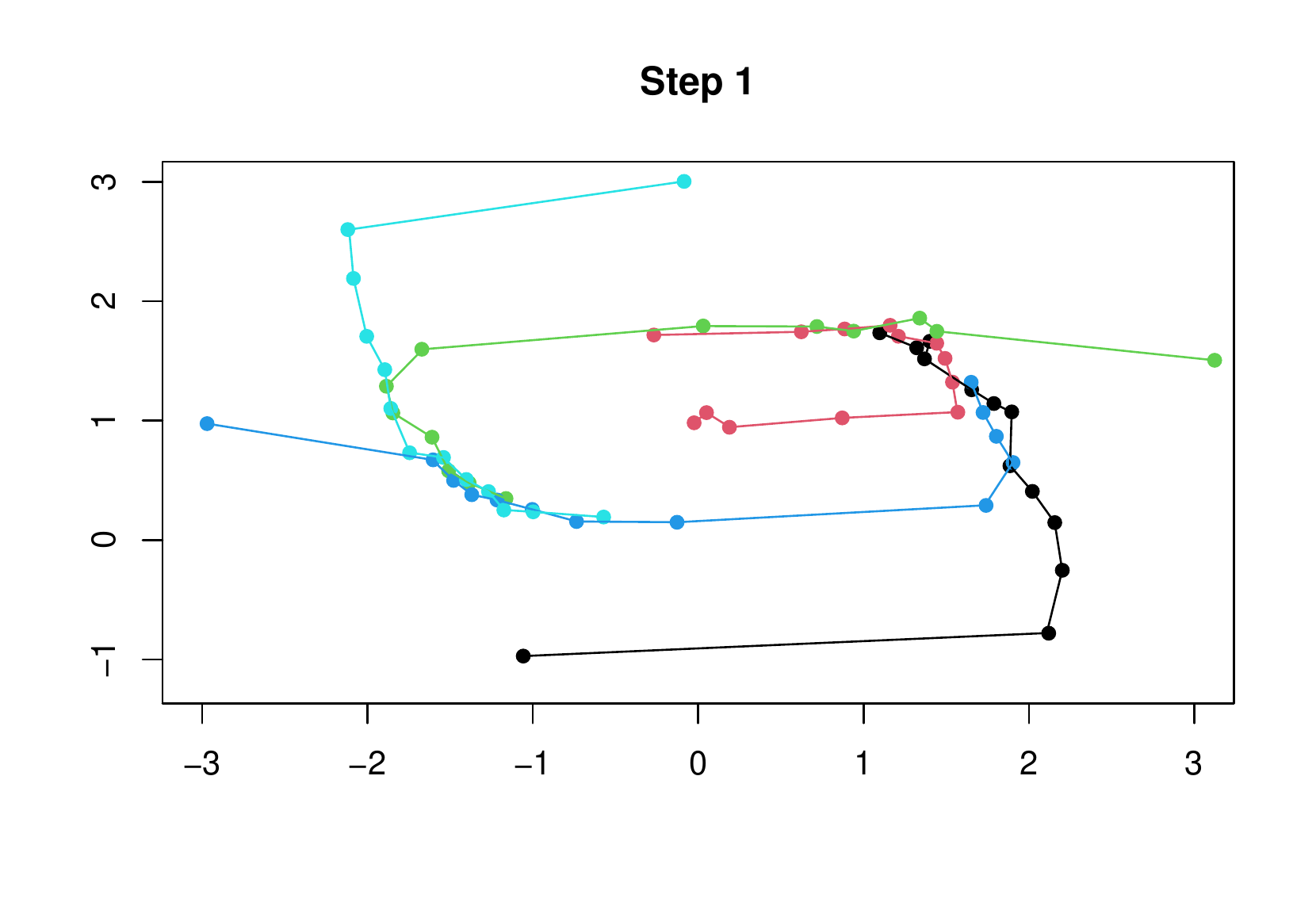}
    \includegraphics[width=0.48\textwidth]{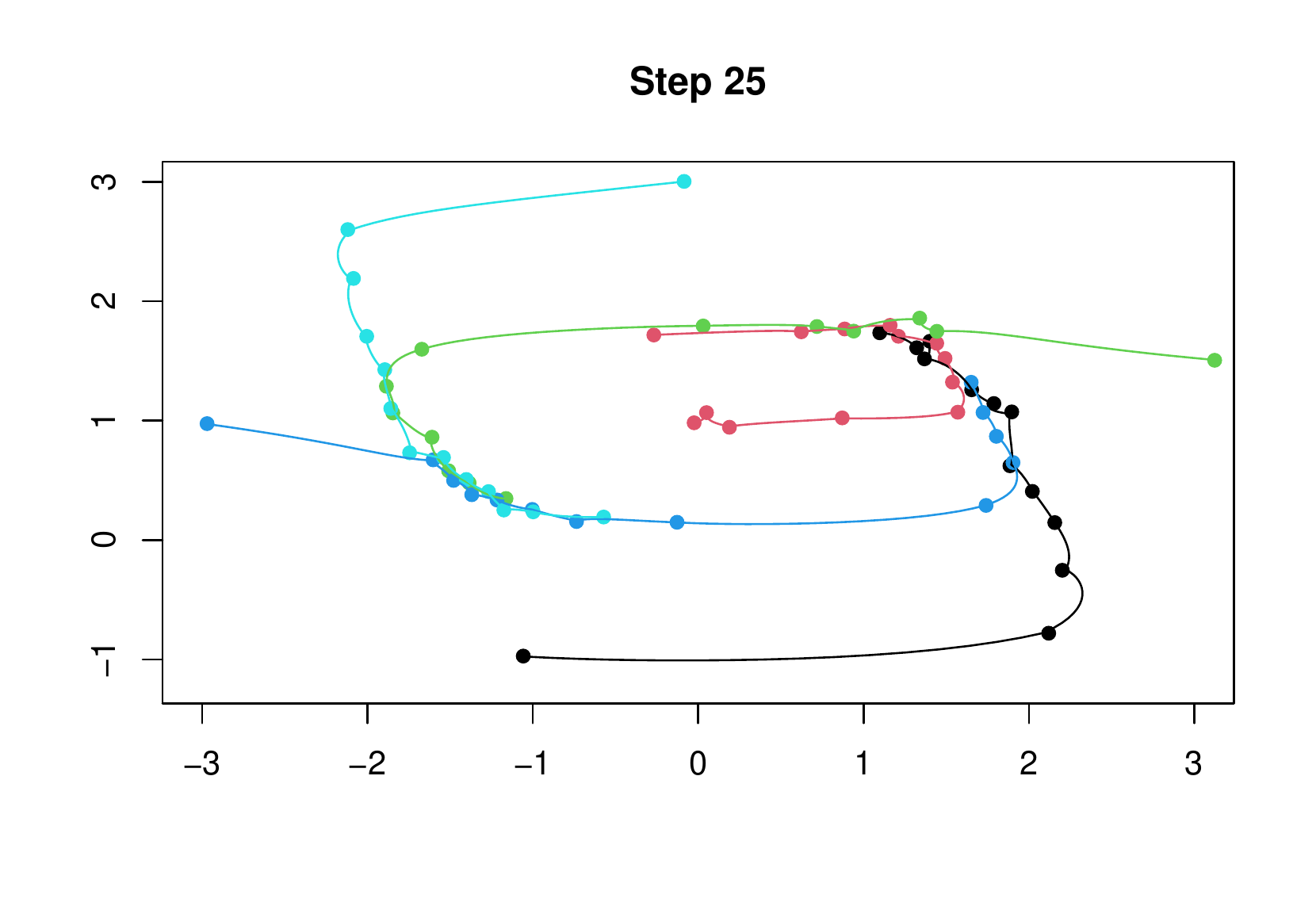}
    \includegraphics[width=0.48\textwidth]{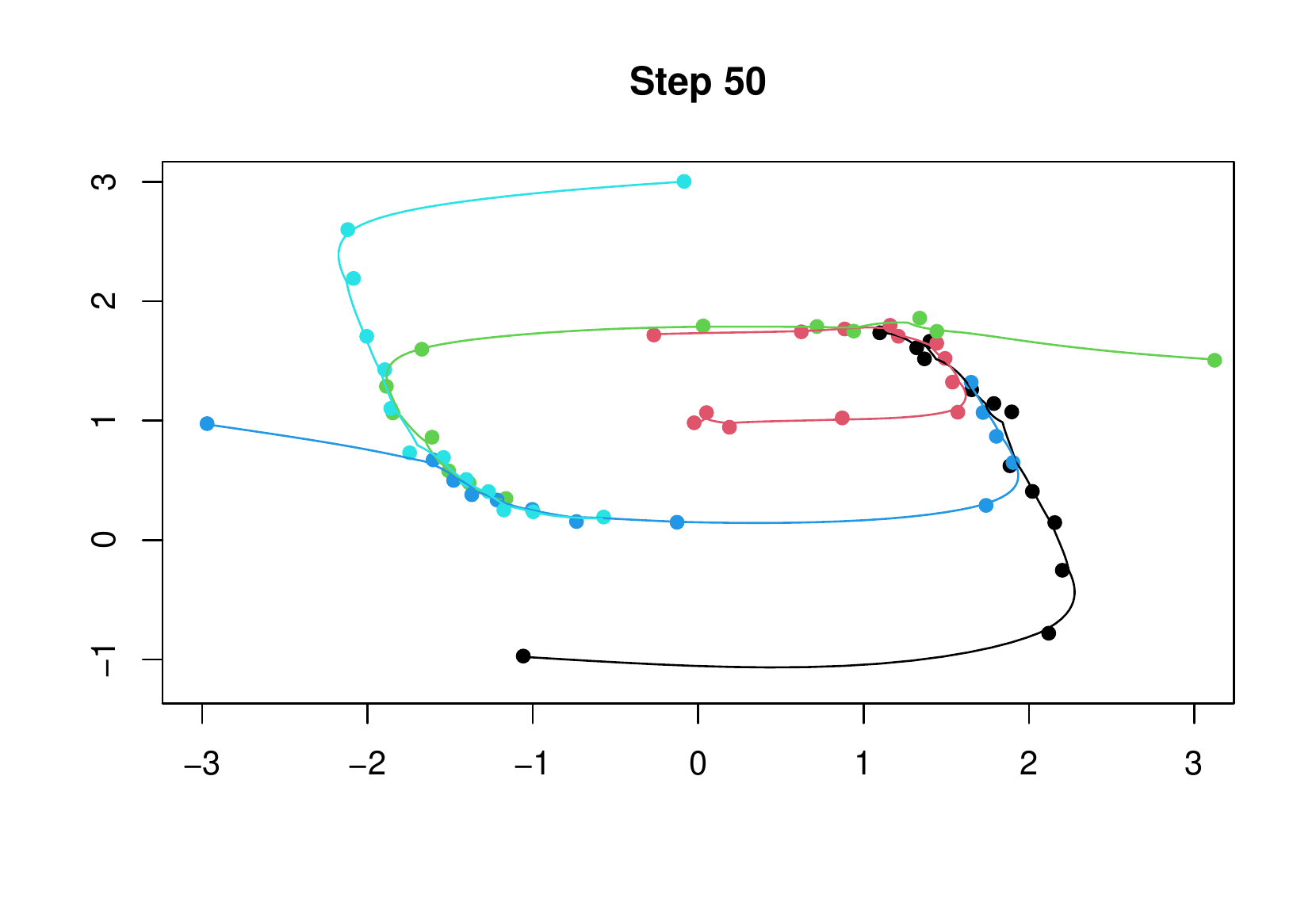}
    \includegraphics[width=0.48\textwidth]{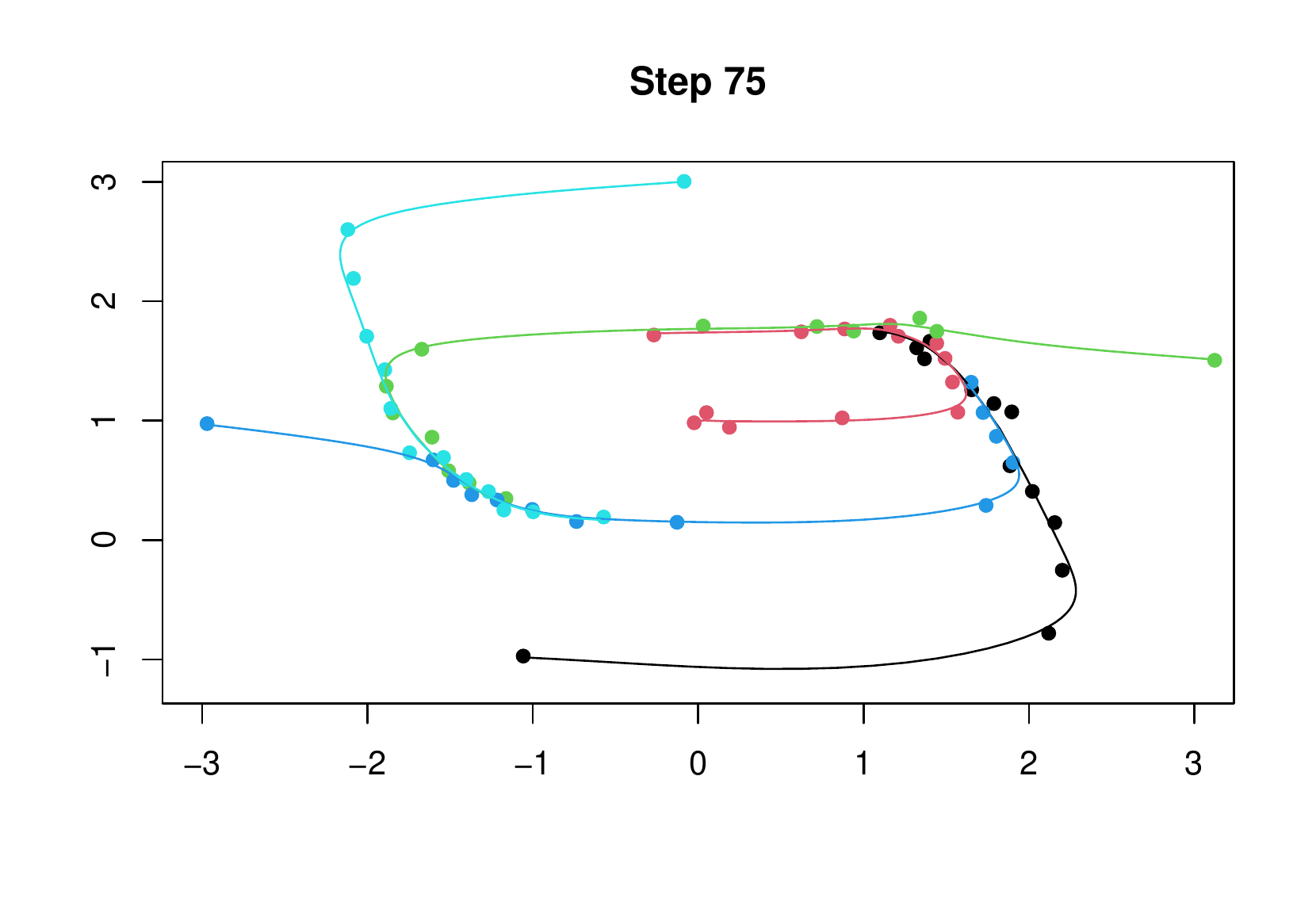}
    \caption{Illustration of the ODE-RKHS Algorithm:  The dots show the observations. The estimated trajectories are shown with lines and curves with corresponding colors. Steps $i$=1,25,50, and 75 are shown from left to right and from top to bottom   }
    \label{fig:run}
\end{figure}

We generated a set of 50 noiseless trajectories.  There were 201 observations per trajectory, one for each $.1$ increment in time.  To generate the training sets, we added samples of Gaussian noise to these fifty trajectories.  There were five levels of noise, with respective standard deviation $\sigma \in \{0.120, 0.365, 0.610, 0.855, 1.100\}$. Next, we generated a single test set of 100 trajectories without noise, again with 201 observations per trajectory separated by $.1$ time increments. Testing consisted of computing predicted trajectories starting at the initial condition of the test trajectories and computing the following error measurement  
\begin{equation}
\label{eq:Err}
    \text{Err} = \sqrt{\sum_{i=2}^{n}(t_{i}-t_{i-1})\|y_i - \hat{y}_i\|^2}
\end{equation}
where $t_i$ refers to the $i^{th}$ observation time, $y_i$ to the $i^{th}$ observation of the test trajectory, $\hat{y}_i$ to the $i^{th}$ point of the predicted trajectory and $n$ to the number of observations in the trajectory.  
We ran experiments with the same training, validation and test sets for all the algorithms. The results are summarized in figure \ref{fig:FHN_results}. The graph shows the average Err over the five datasets corresponding to each level of noise.  The error bars are the standard error of the mean of Err for each of the five datasets. Overall, the performances decrease with increased level of noise, as expected. We noticed that in most cases, ODE-RKHS or npODE are the best performing algorithms. The FHN is a polynomial system. This might explain why SINDy polynomial (in red) performs better than SINDy Fourier (in green).  The gradient descent algorithm is constantly in the higher range of performances.  EDMD did somewhat poorly, as did KAF.        
\subsection{Lorenz data} 
\label{sec:Lorenz}
Our next experiment was on the Lorenz system defined by the equations
\begin{equation}
\begin{split}
\dot{x} &= 10(y-x)\\
\dot{y} &= x(28-z) - y\\
\dot{z} &= xy -\frac{8}{3}z
\end{split}
\end{equation}
We generated 50 noiseless trajectories with 201 observations per trajectory, each separated by a $0.01$ increment in time.  Next, we generated samples of Gaussian noise with levels $\sigma \in \{0.5,1.2,1.9,2.6,3.3\}$.  We generated five noise samples for each noise level and added these to the noiseless trajectories to generate the training sets.  Then we generated a single test set consisting of 100 trajectories, each with 201 observations at $0.01$ time increments.

Err was measured only for the first $0.2$ units of time.  This is because we found that the predicted trajectories for all methods diverged from the true trajectories at about this time.  The results are summarized in figure \ref{fig:exp_error_plots}.  The values are the average of Err over the five datasets corresponding to each level of noise.  The error bars are the standard error of the mean of Err for the five datasets at each noise level.

EDMD performed the best, followed by ODE-RKHS.  SINDy with polynomials did well on the low-noise settings.  This could be because the Lorenz system is a polynomial system.  KAF, SINDy with Fourier features, and gradient descent all did poorly.  The method npODE seemed to break as the noise increased.  L-BP also performed well on the low-noise settings, but struggled when the noise was increased.
{
\subsection{Lorenz96}
The Lorenz96 data arises from \cite{75462}. The chaotic system is defined for $n=6$ dimension by:
\begin{equation}
\dot x_k = -x_{k-1}x_{k-1}+x_{k+1}x_{k-1}-x_k+F , k=1 \ldots 6
\end{equation}
We have selected $F=8$. Indices wrap-around so that $x_{-1} = x_{6}$ and $x_7=x_1$.

The performances of the proposed method, as well as the six comparators, are presented in figure \ref{fig:exp_error_plots}. The ODE-RKHS performs well or better than the comparative methods with noisy data. There is no comparative method that is better uniformly among these three test cases.

}
\begin{figure}[!htb]
    \centering
    \subfigure[FHN]{
    \includegraphics[width=0.4\textwidth]{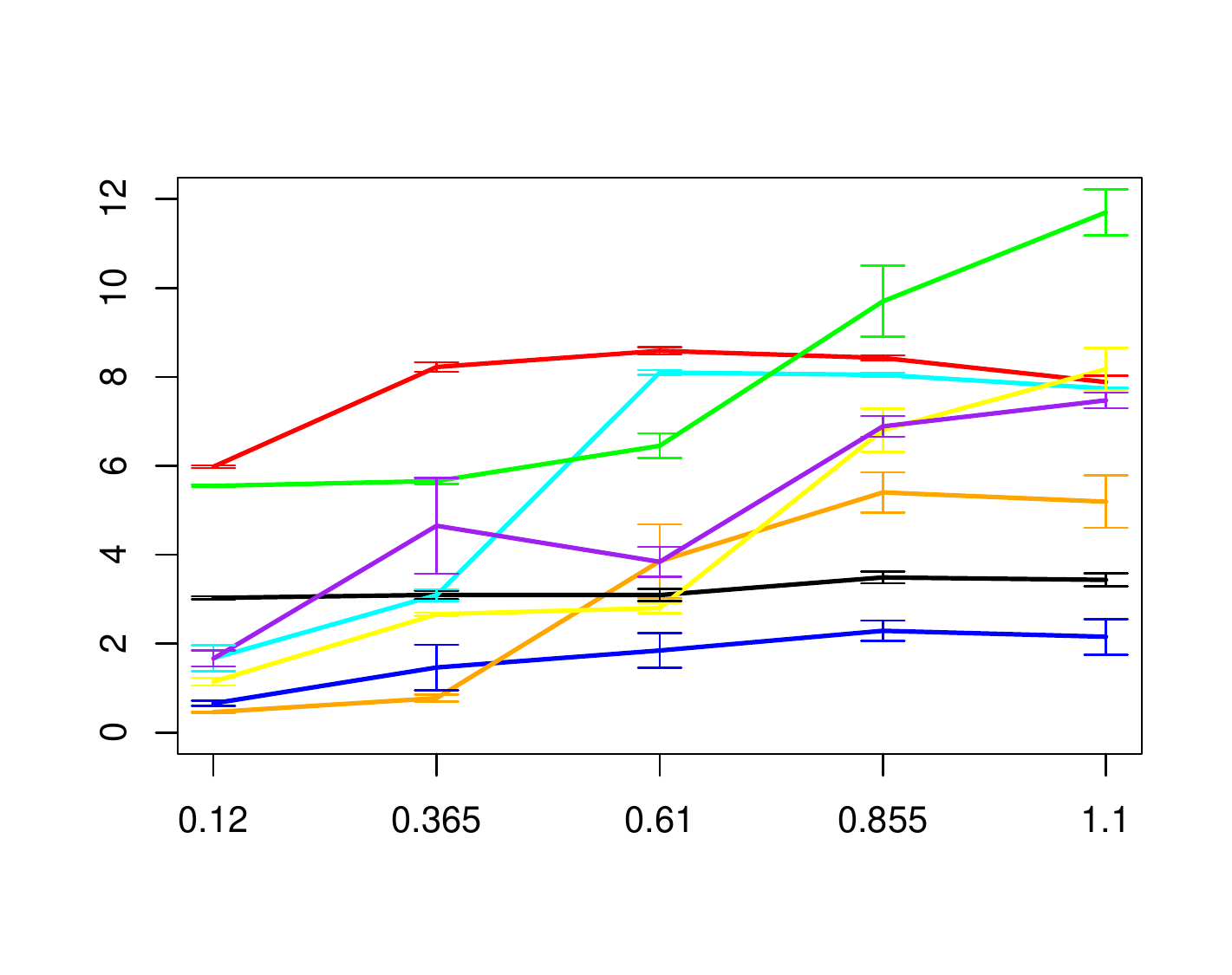}
   }
    \subfigure[Lorenz63]{
    \includegraphics[width=0.4\textwidth,height=.32\textwidth]{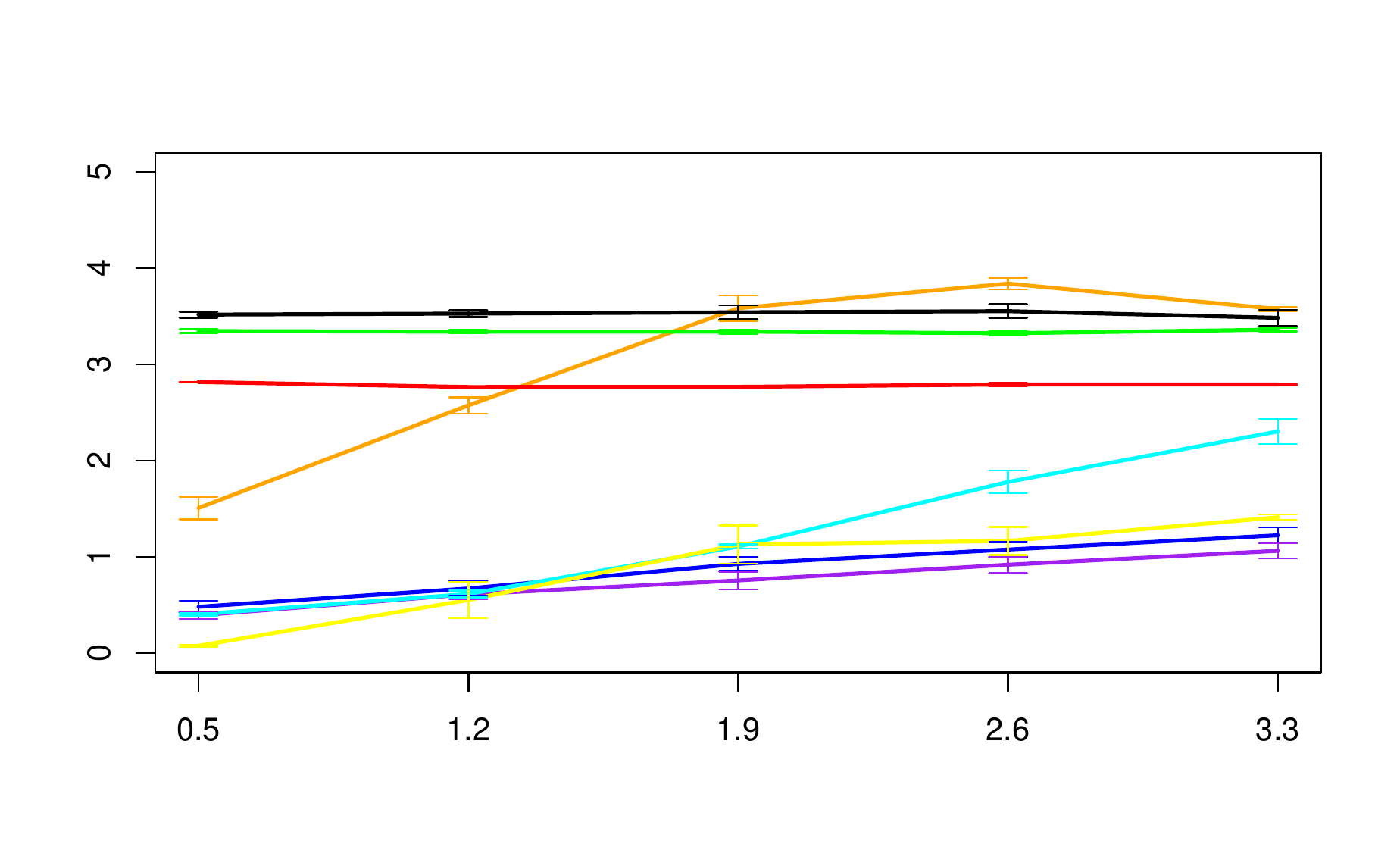}}\\
    \subfigure[Lorenz96-6]{
    \includegraphics[width=0.4\textwidth,height=.4\textwidth]{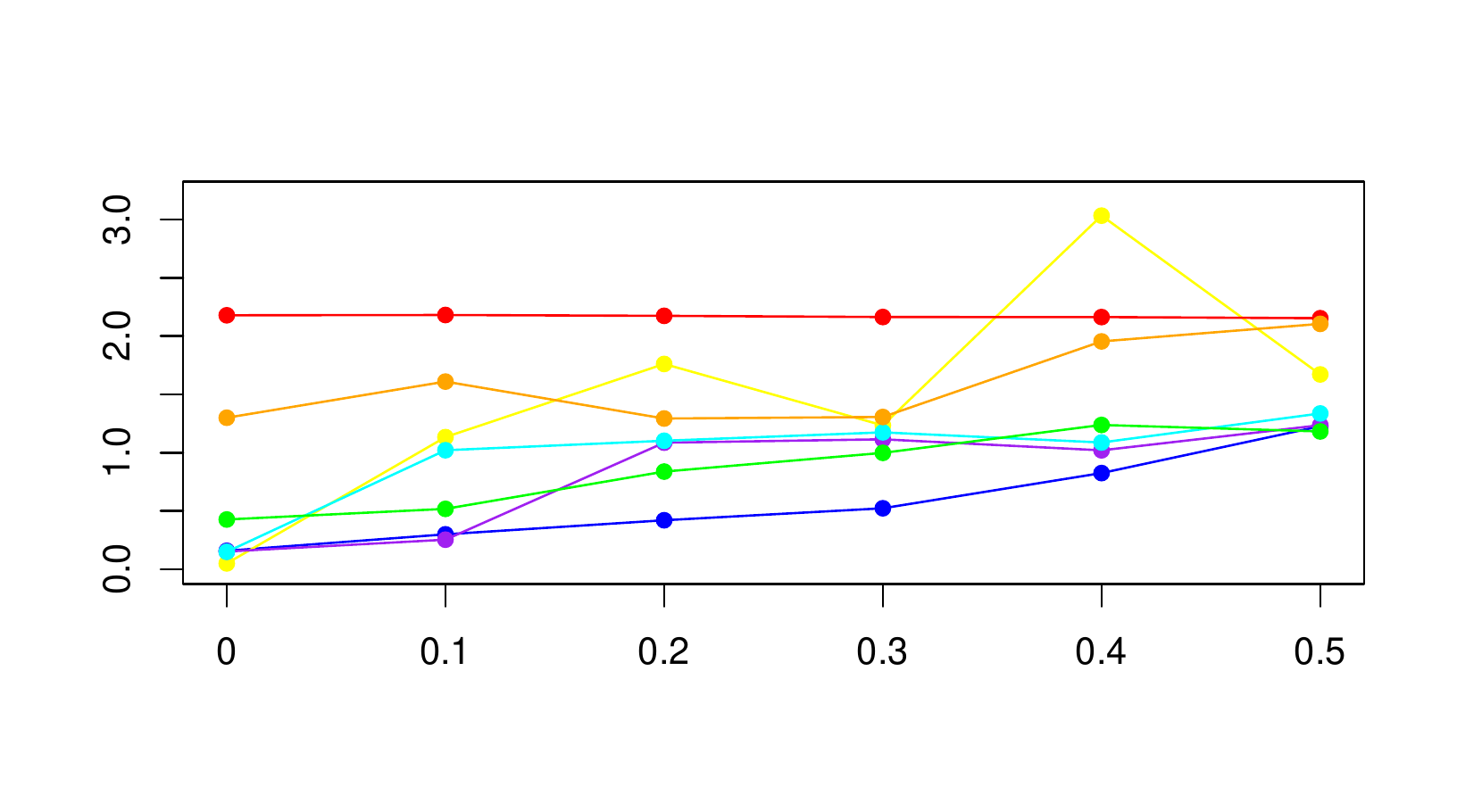}}
    \subfigure[Legend]{
    \includegraphics[width=0.4\textwidth,height=.4\textwidth]{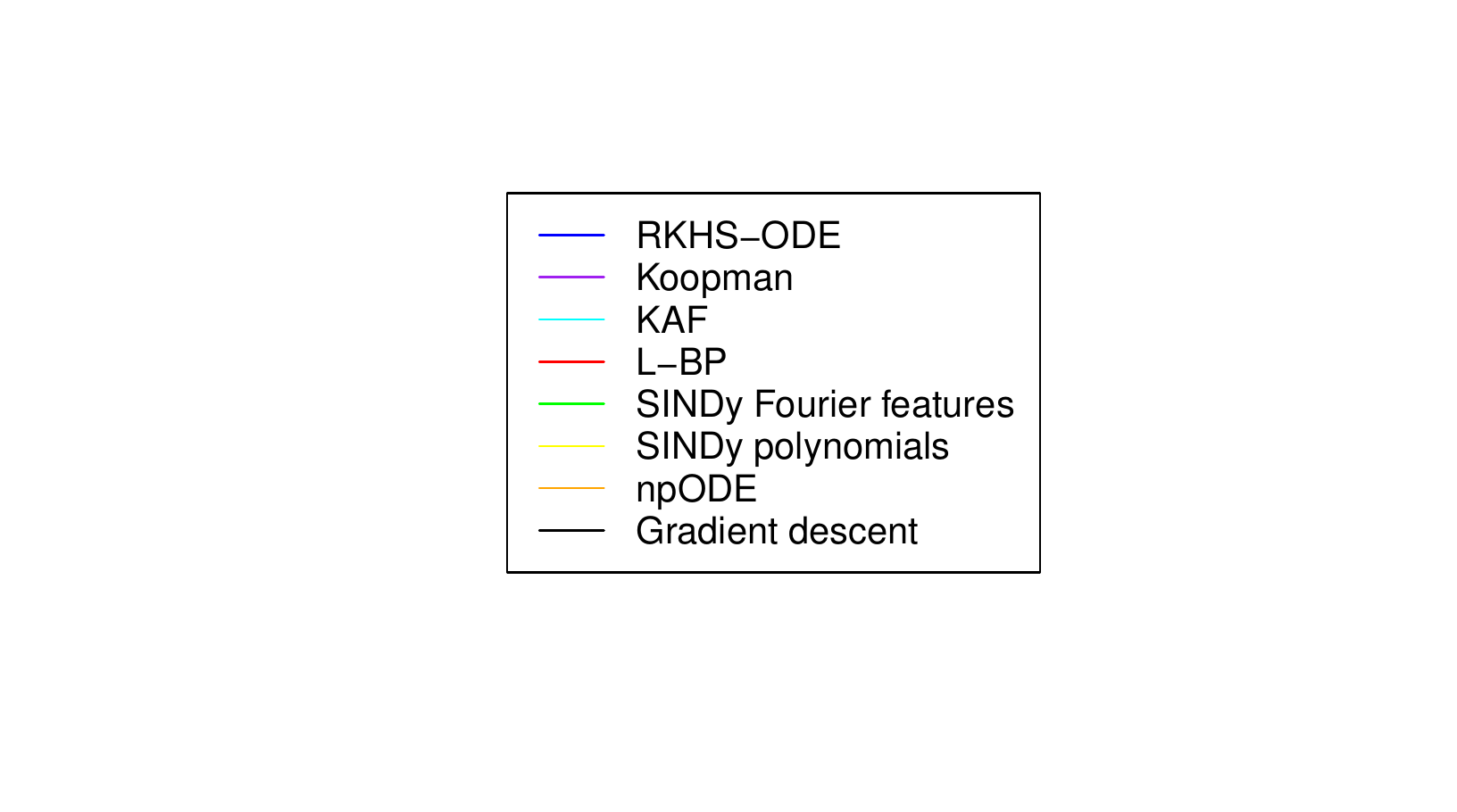}}
    
    \caption{{Analysis of the performances of the RKHS-ODE algorithm on the simulated data. $x$-axis: standard deviation of the Normal (Gaussian) centered noise added to the data. $y$-axis: mean squared difference between the true (noiseless) and estimated trajectories, where the mean is taken over the time points of the true trajectory (averaged over five independent datasets for FHN and Lorenz63).
     } }  
    \label{fig:exp_error_plots}
\end{figure}

\subsection{The accumulation of Amyloid in the cortex of aging subjects}
\label{sec:Amyloid}
The accumulation of Amyloid in the brain is believed to be one of the earliest pathological mechanisms of Alzheimer's disease, beginning more than a decade before the onset of clinical symptoms, see \cite{murphy2010alzheimer}. 

 Based on observations from several longitudinal Amyloid positron emission tomography (PET) studies, it is believed that the rate of Amyloid accumulation is closely associated with the level of Amyloid at the same age, see \cite{vernhet2020modeling}. We develop a principled mathematical model capturing this phenomenon and use it to predict the accumulation of Amyloid across individuals longitudinally. 
 
 We used (PiB) PET scans from the Wisconsin Registry for Alzheimer’s Prevention (WRAP) to assess global Amyloid burden, measured by the Distribution Volume Ratio (DVR)\footnote{The data used for this experiment has been obtained from the Wisconsin Registry for Alzheimer's Prevention. See \url{https://wrap.wisc.edu/}. A request for accessing this data can be initiated from this website.}. The number of subjects in this study is $n=179$, with $3.06$ visits on average, over an average span of 6.84 years. We fit the model in \eqref{eq:model} to the posterior cingulum, precuneus and gyrus rectus DVRs, averaging the left and right DVR in each case. These regions are known to show Amyloid accumulation early in the disease process. We use the Multi Trajectories Penalty method for ODE-RKHS described in Alg. \ref{alg:ODE-RKHS Multi} with $d=3$, and a Gaussian kernel. For each coordinate, we chose a bandwidth equal to 20\% of the range of the data.
 The time step used was $h=.1$ years.  We set $\gamma = 1$ and fit $\lambda,\rho$ using a validation set consisting of 20 percent of the training data. We set a maximum of $S=500$ iterations and used the early stopping criterion of stopping when the ratio $||f^{(s+1)}-f^{(s)}||/||f^{(s)}||$ was less than $\epsilon = 10^{-3}.$  Initialization of $f_0$ was done via gradient matching, as in \eqref{eq:GradientMatching}.  Figure \ref{fig:Amyloid} provides a visualization of the trajectories estimated using RKHS-ODE super-imposed (same color) with the data. This shows that the estimated trajectories are qualitatively accurate.     
We set aside 25 percent (rounded) of the data for testing.
Prediction was performed using Euler integration starting at the first observed time-point for this subject.  We computed the error for every subject as in the FHN experiment.

We compared with the predictions obtained with the other algorithms in table \ref{tab:amyloid}. We found that ODE-RKHS, SINDy polynomial and Fourier, and Gradient descent performed comparably for this data, while EDMD and npODE are not as accurate. npODE performed very well for the FHN data and but it is the worst performing here. A possible explanation is the dimension of the problem, here 3 instead of 2 for the FHN. It might be that some fine tuning of the npODE algorithm, for example increasing the number of inducing points would increase the performance. 

ODE-RKHS performs consistently among the best algorithms.

\begin{figure}[!htb]
    \centering
    \includegraphics[width=0.31\textwidth]{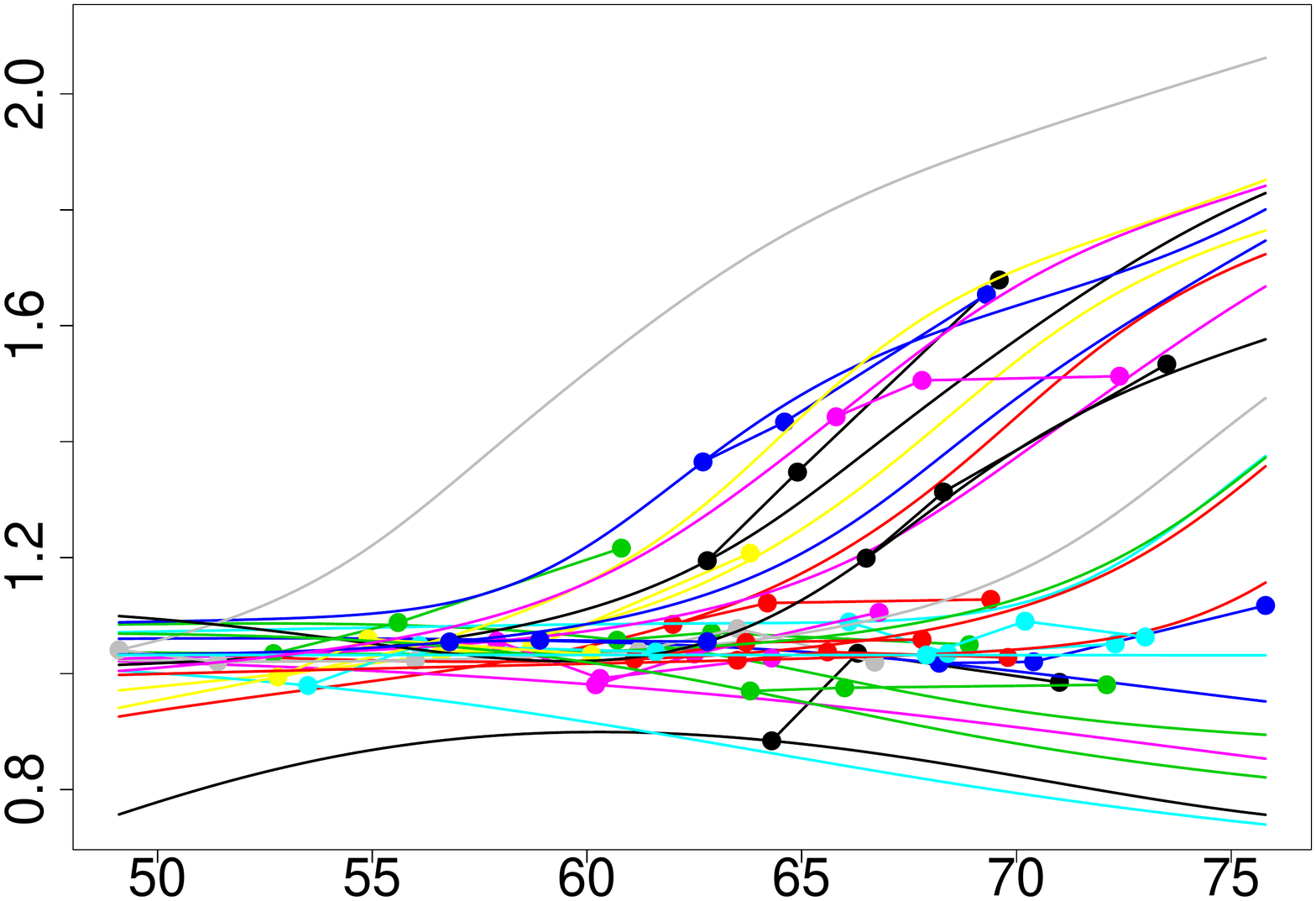}
    \includegraphics[width=0.31\textwidth]{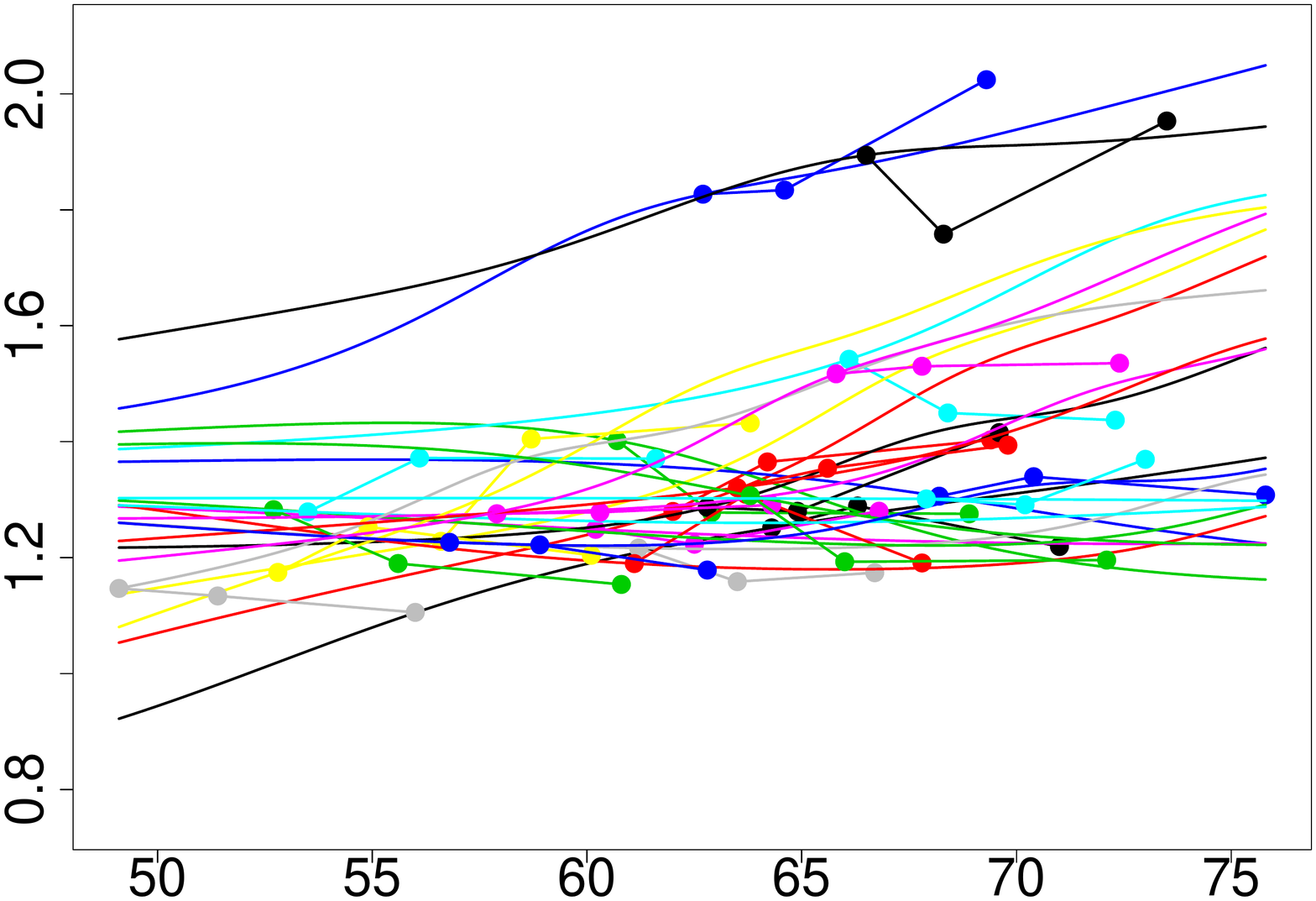}
    \includegraphics[width=0.31\textwidth]{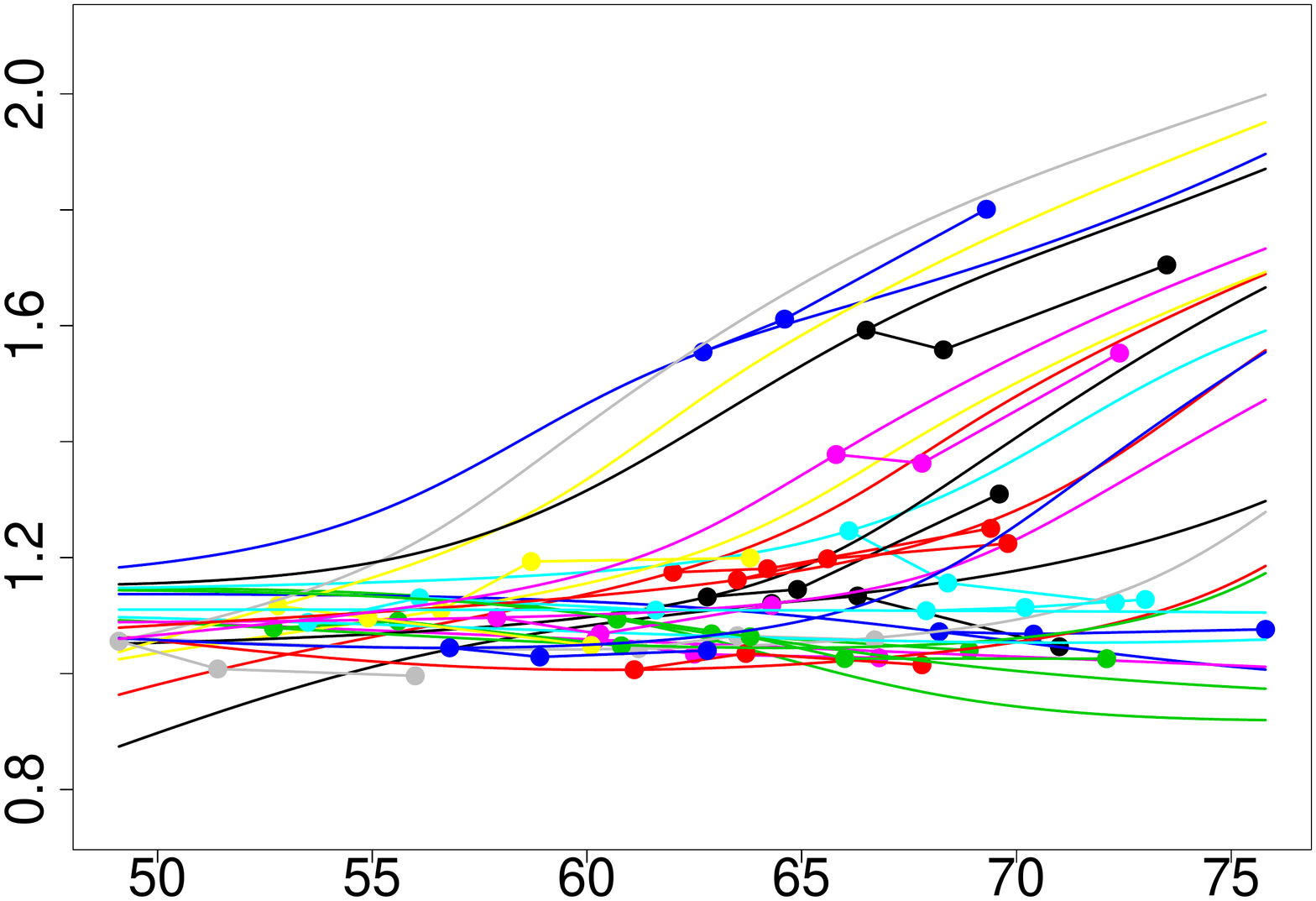}
    \caption{Amyloid prediction experiment. 
    Horizontal axis is in years. Vertical axis corresponds to DVR.  The left-most image corresponds to the gyrus rectus, the middle to the cingulum and the right to the precuneus. }  
    \label{fig:Amyloid}
\end{figure}
\begin{table}[htb]
\begin{center}
\begin{tabular}{lc}
Alg & Error   \\\hline
ODE-RKHS & {\bf 0.34$^{***}$}
\\\hline
npODE & 0.59
\\\hline
SINDy polynomial & {\bf 0.34$^{***}$}
\\\hline
SINDy Fourier & {\bf 0.37$^{***}$}
\\\hline
Gradient descent & {\bf 0.34$^{***}$}
\\\hline
EDMD & 0.53$^*$\\
\hline
KAF & 0.84\\
\hline
L-BP & 0.40$^{**}$\\
\hline
\end{tabular}
\end{center}
\caption{\label{tab:amyloid}Results for Amyloid data.  Stars indicate the number of methods the result is significantly better than as measured by the Wilcoxon signed-rank test at $\alpha=5\%$.}
\end{table}

\section{Discussion}
\label{sec:discussion}
We proposed an algorithm for learning non-parametric ODEs assuming that the function $f$ generating the vector field in $\RR^d$ belongs to a vector-valued RKHS with a kernel satisfying certain regularity conditions. The data input of the algorithm consists of noisy observations at different times of multiple trajectories. The algorithm is linear in the number of observations but cubic in their dimension. We proved the consistency of the estimated trajectory, showing that the $L^2$ squared distance between the estimated trajectory and the true one vanishes as more observations are collected. We assessed the algorithm with simulated and real data and obtained results that consistently compare favorably with the state of the art on a wide range of noise levels. 
\section{Acknowledgements}
The work at Portland State University was partly funded using the National Institute of Health RO1AG021155, R01EY032284, and R01AG027161, National Science Foundation  \#2136228, and the Google Research Award ``Kernel PDE''. The funding sources had no involvement in the study design; in the collection, analysis, and interpretation of data; in the report's writing; and in the decision to submit the article for publication.  The material of Galois, Inc. is based upon work supported by the Air Force Research Laboratory (AFRL) and DARPA under Contract No. FA8750-20-C-0534. Any opinions, findings, conclusions, or recommendations expressed in this material are those of the author(s). They do not necessarily reflect the views of the Air Force Research Laboratory (AFRL) and DARPA.

\appendix

\section{Consistency of the estimator of the trajectory}
\label{sec-supp:proof}

\subsection{Assuming we solve the problem without Euler approximation}

This section gives the proof of the theorem presented in section 3 of the main text. We present the proof for $d=1$ since the generalization to multiple dimensions is straightforward. We also present the proof for the case of autonomous systems. 
Keeping the notations of the main text, we make the following assumptions:
 
\begin{itemize}
\item $\mathbf{A_1}$: There exist an $f^* \in H, {||f^*-f_0||}_H \leq R$ and ${|x_0^*|} \leq r$ such that $x^*(0)=x_0^*$ and $\dot{x}^*(t)=f^*(x^*(t))$ for every $0 \leq t \leq T$.
\item  $\mathbf{A_2}$: The noise variables $\epsilon_{j}$ are independent and  bounded by a constant $M_{\epsilon}$, with a variance denoted by $\sigma^2$. (We can assume that the variables are subgaussian instead of bounded if we want to generalize this result)
\item  $\mathbf{A_3}$: The kernel $K$ is $\mathcal{C}^2(\mathbb{R})$ in its first argument (this implies that it is also $\mathcal{C}^2(\mathbb{R})$ in its second argument).
\item $\mathbf{A_4}$: The kernel $K$ satisfies the hypothesis of {lemma} 1.
\end{itemize}

Without loss of generality, we will assume that $f_0=0$ in our proof.

Let $H$ be the RKHS with reproducing kernel $K$. Let $f \in H$  such that ${||f||}_{H} \leq R$. We know using assumption $\mathbf{A_4}$ and lemma 1 that $f$ is uniformly Lipschitz, with a Lipschitz constant that does not depend on $f$ that we denote by $L_1$. Specifically, 
\begin{equation}
\label{UniLip}
|f(x)-f(y)| \leq L_1 |x-y|
\end{equation}
with $L_1=N_KR$
Using \eqref{UniLip}, we will prove the following lemma:

\begin{lemma}
\label{Sol.Bound}
Assuming $\mathbf{A_4}$, consider the set of solutions to the problem 
\begin{equation}
    \frac{\partial x}{\partial t}=\dot{x}=f(x), x(t_0)=x_0
\end{equation}
where $f$ belongs to the RKHS with kernel $K$ , $|x_0| \leq r$ and $t \in [0,T]$. Then any solution $x$ in this set of solutions is bounded by a uniform constant $B_1$ that only depends on $T$, $R$, $L_1$ and $L_3^2:=\sup_{||x||<C}|K(x,x)|$.

Specifically, 
\begin{equation}
|x(t)-x(t_0)| \leq B_1={T L_3 R} e^{L_1 T}
\end{equation}
\end{lemma}

\begin{proof}

We start by taking $f$ in our class of functions and $x_0$ such that $|x_0| \leq r$. We therefore can write:
\begin{align}
|x(t)-x_0| &=\left|\int_0^t(f(x(s))-f(x_0))ds+t f(x_0)\right| \\
     & \leq \int_0^t|f(x(s))-f(x_0)|ds+ t{||f||}_{H} \sqrt{K(x_0,x_0)} \\
     & \leq L_1 \int_0^t|x(s)-x_0|ds+TL_3R
\end{align}

Now denote by $G(t):=|x(t)-x_0|$. If we prove that $G(t)$ is bounded by a constant depending only on $T$, $R$, $L_1$ and $L_3$, we will be done. So far we have:

\begin{equation}
G(t) \leq L_1 \int_0^tG(s)ds+T L_3R
\end{equation}
Denote by $V(t):=\int_0^tG(s)ds$. We have that:
\begin{equation}
V'(t)\leq L_1 V(t)+T L_3R
\end{equation}
which implies:
\begin{equation}
e^{-L_1t}V'(t)-L_1e^{-L_1t} V(t) \leq T L_3R e^{-L_1t}
\end{equation}
Integrating the inequality between  0 and $t$ using the fact that $V(0)=G(0)=0$, we obtain:

\begin{equation}
\exp{(-L_1t)}V(t) \leq \frac{T L_3R}{L_1} (1-e^{-L_1 t})
\end{equation}
 
or, equivalently, 

\begin{equation}
V(t) \leq \frac{TL_3R}{L_1} (e^{L_1t}-1)
\end{equation}

Finally since $V'(t)=G(t)\leq L_1 V(t)+TL_3R$, we have:

\begin{equation}
G(t) \leq {T L_3 R} e^{L_1t} \leq {T L_3 R} e^{L_1 T}
\end{equation}

\end{proof}

Let us now introduce the following notations:

\begin{itemize}
    \item We denote by $x(x_0,f,t)$ the solution to the ODE with derivative $f$ and initial condition $x_0$
    \item $y_i$ is the observed noisy point from the trajectory at time $t_i$.
    \item $x^*(t)$ is the true trajectory evaluated at time $t$
\end{itemize}

We now  proceed with  the following reasoning. We assume that
our trajectory minimizes 

\begin{equation}
\hat{L}(f,x_0):=\sum_{i=1}^m(t_{i+1}-t_i)\left({(x(x_0,f,t_i)-y_i)}^2-\sigma^2\right)
\end{equation}
over $(f,x_0)$ such that $||f||_{H} \leq R$, and $|x_0| \leq r$.
We denote the minimizer by $(\hat{f},\hat{x}_0)$.

When $x_0$ and $f$ are fixed and not data dependent (deterministic), the expected value of $\hat{L}(f,x_0)$ is :
\begin{equation}
L(f,x_0):=\sum_{i=1}^m(t_{i+1}-t_i){(x(x_0,f,t_i)-x^*(t_i))}^2
\end{equation}

Notice that $\mathbf{A_1}$ implies:

\begin{equation}
\label{Min0}
min_{||f||_{H} \leq R, |x_0| \leq r}L(f,x_0)=L(f^*,x_0^*)= \sum_{i=1}^m(t_{i+1}-t_i){(x^*(t_i)-x^*(t_i))}^2=0
\end{equation}

Our goal is to evaluate $L(\hat{f},\hat{x}_0)$ and obtain a generalization bound. We have:

\begin{equation}
\label{eq:loss1}
L(\hat{f},\hat{x}_0)   =L(\hat{f},\hat{x}_0)-\hat{L}(\hat{f},\hat{x}_0)+ \hat{L}(\hat{f},\hat{x}_0)-\hat{L}({f^*},{x_0^*})+\hat{L}({f^*},{x_0^*})-{L}({f^*},{x_0^*})
\end{equation}

And therefore, since the middle term in \eqref{eq:loss1}: $\hat{L}(\hat{f},\hat{x}_0)-\hat{L}({f^*},{x_0^*})<0$, 

\begin{equation}
\label{SupFIneq}
 L(\hat{f},\hat{x}_0)   \leq \sup_{||f||_{H} \leq R, |x_0| \leq r}2|L({f},{x_0})-\hat{L}({f},{x_0})|
\end{equation}

We thus consider the following quantity :

\begin{equation}
\text{Err}:=\sup_{||f||_{H}\leq R, |x_0| \leq r}|\hat{L}(f,x_0)-L(f,x_0)|
\end{equation}

Expanding this quantity we get:

\begin{equation}
\label{SupDev}
\sup_{||f||_{H}\leq R, |x_0| \leq  r}\left|\sum_{i=1}^m(t_{i+1}-t_i)(y_i^2-{x^*(t_i)}^2- \sigma^2-2 x(x_0,f,t_i)(y_i-x^*(t_i))\right|
\end{equation}

Notice that if we replace for a given single  $i$, $y_i=x^*(t_i)+\epsilon_i$ by  $\tilde{y_i}=x^*(t_i)+\tilde{\epsilon}_i$, the quantity of equation \ref{SupDev} will change by a quantity bounded by some  constant   $K_2(t_{i+1}-t_i)$, that we can bound  by $4(B_1+r+M_{\epsilon})M_{\epsilon}+4(B_1+r)M_{\epsilon}$. Therefore, using McDiarmid inequality \cite{doob1940regularity}:

\begin{equation}
\label{McDiarmid}
\mathbb{P}\left(\text{Err}\geq \mathbb{E}(\text{Err})+\epsilon\right) \leq \exp{\left(\frac{-2\epsilon^2}{K_2^2\sum_{i=1}^m{(t_{i+1}-t_i)}^2}\right)}
\end{equation}

We therefore need to provide an upper bound  of $\mathbb{E}(\text{Err})$. For that, we are going to view:

\begin{equation}
|\hat{L}(f,x_0)-L(f,x_0)|= \left|\sum_{i=1}^m(t_{i+1}-t_i)(y_i^2-{x^*(t_i)}^2- \sigma^2-2 x(x_0,f,t_i)(y_i-x^*(t_i)))\right|
\end{equation}

as a stochastic process indexed by $x$, where $x \in \mathcal{X}$: Set of all solutions $x(f,x_0,.)$ for all ${||f||}_H \leq R$ and $|x_0| \leq r$. In other words, we view the process $|\hat{L}(f,x_0)-L(f,x_0)|$ indexed by $f$ and $x_0$ as:

\begin{equation}
|\hat{L}(x)-L(x)|
\end{equation}

where $x \in \mathcal{X}$ is some $x(f,x_0,.)$. Notice that $\text{Err}$ is also:

\begin{equation}
\sup_{x \in \mathcal{X}}|\hat{L}(x)-L(x)|
\end{equation}

Notice that $x$ is a subset of continuous functions defined on $[0,T]$. Therefore we can equip $\mathcal{X}$ with the metric  structure $(\mathcal{X},{||.||}_{\infty})$. We will apply Dudley's inequality (see for e.g \cite{vershynin2018high}, theorem 8.1.3) to bound:

\begin{equation}
\mathbb{E}(\text{Err})=\mathbb{E}\left(\sup_{||f||_{H}\leq R, |x_0| \leq r}|\hat{L}(f,x_0)-L(f,x_0)|\right)
\end{equation}

To apply Dudley's inequality, we are going to use the following lemma.

\begin{lemma}
\label{UniLipschitz}
The solutions $x \in \mathcal{X}$ are Lipschitz with a Lipschitz constant that is uniform over $\mathcal{X}$, i.e, there exists a constant $L_6$ such that for every $x \in \mathcal{X}$, $t \in [0,T]$ and $s \in [0,T]$:
\begin{equation}
|x(t)-x(s)| \leq L_6 |t-s|
\end{equation}
$K_6$ depends on $R,B_1$,$r$ and the kernel $K$.
\end{lemma}

\begin{proof}
Let $x_0$ such that $|x_0| \leq r$ and $f$ such that ${||f||}_{H} \leq R$. We have:
\begin{align}
|\dot{x}(x_0,f,t)| &= |f(x(t))| \\
                   & \leq R \sqrt{\sup_{|x|\leq B_1+r}K(x,x)}
\end{align}
\end{proof}

As a consequence, if we denote by $\mathcal{N}(\mathcal{X},\epsilon)$ the covering number of $\mathcal{X}$ with  a radius $\epsilon$ we have the existence of a constant $L_7$ ($L_7$ only depends on $B_1$,$r$ and $L_6$) such that:

\begin{equation}
\label{Covering}
\mathcal{N}(\mathcal{X},\epsilon) \leq  \exp{\left(\frac{L_7}{\epsilon}\right)},
\end{equation}

where we used a known upper bound that can be found for example in   \cite{vershynin2018high} (exercise 8.2.7) on the covering number of uniformly bounded Lipschitz continuous functions defined on a finite interval.

Using this result combined with Dudley's inequality,  we obtain the existence of a constant $L_8$ (depending only on $L_7$) such that:

\begin{proposition}
\label{Dudley}
\begin{equation}
\mathbb{E}(\textnormal{Err})  \leq L_8 \sqrt{\sum_{i=1}^m{(t_{i+1}-t_{i})}^2}
\end{equation}
\end{proposition}

\begin{proof}
Apply Dudley's inequality to $\text{Err}$ using  inequality \eqref{Covering} and the fact that the diameter of $\mathcal{X}$ is finite bounded by $2(B_1+r)$ and that for every $M<\infty$
\begin{equation}
\int_{0}^M \sqrt{\log\left(\mathcal{N}\left(\mathcal{X},\epsilon\right)\right)}d\epsilon \leq \int_{0}^M \sqrt{\log\left(\exp{\left(\frac{K_7}{\epsilon}\right)}\right)}d\epsilon < \infty
\end{equation}
\end{proof}





As a consequence, using \eqref{McDiarmid} and theorem \eqref{Dudley},  we obtain the following inequality:

\begin{equation}
\label{PreConc}
\mathbb{P}\left(\text{Err} \geq   L_8 \sqrt{\sum_{i=1}^m{(t_{i+1}-t_{i})}^2}+\epsilon \right) \leq \exp{\left(\frac{-2\epsilon^2}{K_2^2\sum_{i=1}^m{(t_{i+1}-t_i)}^2}\right)}
\end{equation}

Using inequalities \eqref{SupFIneq} and \eqref{PreConc} we finally obtain the following theorem:

\begin{theorem}
\label{MainTheo1}
With assumptions $\mathbf{A_1}, \mathbf{A_2}$, $\mathbf{A_3}$ and $\mathbf{A_4}$, there exist   constants $L_9$  and $K_{2}$ depending only on $R$, $r$, $T$, $M_{\epsilon}$ and the kernel $K$ such that for every $\epsilon$:
\begin{equation}
\mathbb{P}\left(  L(\hat{f},\hat{x}_0) \geq  L_9 \sqrt{\sum_{i=1}^m{(t_{i+1}-t_{i})}^2}+\epsilon \right)  \\ \leq  \exp{\left(\frac{-2\epsilon^2}{K_2^2\sum_{i=1}^m{(t_{i+1}-t_i)}^2}\right)}
\end{equation}
\end{theorem}

\subsection{Including the Euler approximation}

In reality, the solution (trajectory) that we propose for every $f$ and $x_0$ is not $x(x_0,f,.)$  the solution of the ODE but $\tilde{x}(x_0,f,h,.)$, the solution obtained with an Euler's method of time step $h$. The idea is to use the fact that under some sufficient conditions, we know how to bound the error between Euler's method and the true solution. For example,   we know that if $f$ is Lipschitz with a Lipschitz constant $K_1$ and the solution $x(x_0,f,.)$ is $\mathcal{C}^2$ with a constant $K_{11}$ such that:

\begin{equation}
\label{CondDer}
x''(x_0,f,t) \leq L_{11},  \forall 0 \leq t \leq T
\end{equation}

then we have the following global truncation error bound \cite{atkinson2008introduction}:

\begin{equation}
\label{EulerErr}
\max_{1\leq i \leq n}|x(x_0,f,t_i)-\tilde{x}(x_0,f,h,t_i)| \leq  \frac{h L_{11}}{2L_1}\left(\exp^{L_1T}-1\right)
\end{equation}

We already showed that $f$ is Lipschitz with some constant $L_1$. To ensure the condition of inequality \eqref{CondDer}, notice that:

\begin{equation}
x''(x_0,f,t)=f(x(x_0,f,t))f'(x(x_0,f,t))
\end{equation}

Since we already showed that the solutions $x(x_0,f,.)$ are uniformly bounded by $B_1+r$, it is sufficient to ensure that $f$ is $\mathcal{C}^1$. This is true if we assume that our kernel $K$ is $\mathcal{C}^2$ and hence \eqref{EulerErr} will be insured. \\

Taking into account the Euler approximation and the error bound, the steps of the consistency proof are identical only with the following important difference in equation \eqref{Min0} from the previous section

\begin{equation}
   \min_{||f||_{H} \leq R, |x_0| \leq r}L(f,x_0)   \leq L(f^*,x_0^*)  
\end{equation}
with 
\begin{equation}
  L(f^*,x_0^*) = \sum_{i=1}^m(t_{i+1}-t_i){(\tilde{x}^*(t_i,h)-x^*(t_i))}^2 
  \leq   \frac{h^2 {L_{11}}^2 T }{4L_1^2}{\left(\exp^{L_1 T}-1\right)}^2:=L_{12}
\end{equation}

With this modification, theorem \ref{MainTheo1} becomes:

\begin{theorem}
\label{MainTheo2}
Assuming $\mathbf{A_1}, \mathbf{A_2}$, $\mathbf{A_3}$ and $\mathbf{A_4}$,  there exist   constants $K_{2}$, $L_{12}$ and $L_{13}$ depending only on $R$, $r$, $T$,$M_{\epsilon}$ and the kernel $K$ such that for every $\epsilon$:
\begin{equation}
\mathbb{P}\left(  L(\hat{f},\hat{x}_0) \geq   L_{13} \sqrt{\sum_{i=1}^m{(t_{i+1}-t_{i})}^2}+h^2L_{12}+\epsilon \right) \leq \\ \exp{\left(\frac{-2\epsilon^2}{K_2^2\sum_{i=1}^m{(t_{i+1}-t_i)}^2}\right)}
\end{equation}
\end{theorem}

\subsection{$L^2$ squared distance between the true solution and the estimated trajectory}

In reality $L(\hat{f},\hat{x}_0)$ is an approximation of the $L^2$ norm squared

\begin{equation}
{||x(\hat{f},\hat{x}_0,\cdot)-x^*(\cdot)||}_{L_2}^2:=\int_{0}^T {\left(x(\hat{f},\hat{x}_0,t)-x^*(t)\right)}^2 dt
\end{equation}

Since we proved that the solutions are uniformly bounded by $(B_1+r)$ and $\dot{x}$ is bounded by $L_6$, we have $t \rightarrow {\left(x(\hat{f},\hat{x}_0,t)-x^*(t)\right)}^2$ is Lipschitz with Lipschitz constant $8(B_1+r)L_6$ (we just bound the norm of the derivative). Therefore:

\begin{equation}
    |{||x(\hat{f},\hat{x}_0,\cdot)-x^*(\cdot)||}_{L_2}^2-L(\hat{f},\hat{x}_0)| \leq  8 (B_1+r)L_6 \sum_{i=1}^m{(t_{i+1}-t_i)}^2
\end{equation}

Which proves theorem 2 of the main text.



\section{Kernels}
\label{sec-supp:kernels}
We are interested in listing kernels that satisfy Lemma 1, and thus can be used to model ODEs admitting a single solution. There are cases when one can directly verify the hypothesis of Lemma 1. In the case of translation invariant kernels, one can use the Bochner theorem to provide a sufficient condition as explained in the next section.  
\subsection{Translation invariant kernels}
We consider translation invariant scalar positive definite kernels over $\RR^d$, that is kernels for which 
\begin{equation}
    k(u,v)= h(u-v), u,v \in \RR^d
\end{equation}
The Bochner theorem provides a characterization of translation invariant kernels. Specifically, there exists a probability density $q$ with respect to the Lebesgues measure over $\RR^d$ such that 
\begin{equation}
    h(x) = h(0)\int_{\RR^d}e^{ix^Ty}q(y)dy
\end{equation}
Furthermore, since we restrict our attention to real-valued kernels, 
\begin{equation}
    h(x) = h(0) \int_{\RR^d} cos(x^Ty)q(y)dy
\end{equation}
The gradient of $h$ is then formally the vector of length $d$
\begin{equation}
    \nabla h(x) = -h(0)\int_{\RR^d}ysin(x^Ty)q(y)dy 
\end{equation}
and the Hessian of $h$ is formally the matrix
\begin{equation}
    \nabla \nabla h(x) = -h(0)\int_{\RR^d}(yy^T)cos(x^Ty)q(y)dy
\end{equation}
Translation invariant kernels that satisfy Lemma 1 are such that 
\begin{equation}
Q(x)==c||x||^2 + 2(h(x)-h(0))\geq 0    
\end{equation}
for some constant $c>0$ and for any $x,y \in \RR^d$. Notice that $Q(0)=0$. Next, since $\nabla h(0)=0$, $\nabla Q(0)=0$. Moreover, 
\begin{equation}
    \nabla \nabla Q(x) = 2cI + 2 \nabla \nabla h(x)
\end{equation}
where $I$ is the identity matrix. Next, since $\nabla \nabla Q$ is a symmetric matrix, it has real eigenvalues. Suppose these eigenvalues are bounded uniformly from below. In that case, one can choose a constant $c$ large enough such that $\nabla \nabla Q(x)$ is positive definite for each $x\in \RR^d$ which implies that $Q$ is convex and since $Q(0)=0$ and $\nabla Q(0)=0$, $Q(x) \geq 0$ for each $x \in \RR^d$ and the conditions for Lemma 1 are satisfied.  
A sufficient condition for this to happen is that all the coordinates of $\nabla \nabla h$ are bounded, i.e., for each $i \in \{1,\ldots,d\}$, $E[Y_i^2]<\infty$, where $Y_i$ is a random variable with density $q_i$, the $i^{th}$ marginal of $q$. 
{
\subsection{Explicit Kernels:}

We begin by observing the condition 
\begin{equation}
d^2_{K_{ii}}(u,v) \leq N^2_K|u-v|^2, \forall u,v \in \mathbb{R}^d, i\ = 1,...,d
\end{equation}
is equivalent to the condition:
\begin{equation}
\sum_{i=1}^{d}d^2_{K_{ii}}(u,v) \leq N^2|u-v|^2
\end{equation}
Consider the case where $K$ is an explicit kernel. That is to say there exists a finite (p) dimensional feature space and a mapping $\Phi:\mathbb{R}^d\rightarrow \mathbb{R}^{p\times d}$ for which:
\begin{equation}
K(u,v) = \Phi(u)^T\Phi(v)
\end{equation}
The Fourier random features used in our experiments fall in this category.
\begin{lemma}
\label{lem:LipContExp}
\begin{equation}
\sum_{i=1}^{d}d^2_{K_{ii}}(u,v) = \|\Phi(u) - \Phi(v)\|^2_{\mathcal{F}}
\end{equation}
Where $\mathcal{F}$ is the Frobenious norm.\\
\textbf{Proof:} 
\begin{align}
\sum_{i=1}^{d}\left(k_{i,i}(u,u) - 2k_{i,i}(u,v) + k_{i,i}(v,v)\right) &= \sum_{i=1}^{d}e_i^T\Phi(u)^T\Phi(u)e_i - 2e_i^T\Phi^T(u)\Phi(v)e_i + e_i^T\Phi(v)^T\Phi(v)e_i \\
&= \sum_{i=1}^{d}e_i^T\left\{\Phi(u)^T\Phi(u) - \Phi(u)^T\Phi(v)-\Phi(v)^T\Phi(u)+\Phi(v)^T\Phi(v)\right\}e_i \\
&= \sum_{i=1}^{d}e_i^T\left(\Phi(u)-\Phi(v)\right)^T\left(\Phi(u)-\Phi(v)\right)e_i \\
&= Trace\left((\Phi(u) - \Phi(v))^T(\Phi(u)-\Phi(v))\right) \\
&= \|\Phi(u)-\Phi(v)\|^2_{\mathcal{F}}
\end{align}
\end{lemma}
Therefore, for explicit kernels, we conclude that the condition of lemma 1 is equivalent to the condition that the features are Lipschitz continuous with respect to the Frobenious norm.
}

\subsection{Examples of kernels which satisfy the assumptions of  {lemma 1}}
Let us notate 
\begin{equation}
    P(u,v) = K_1(u,u)+K_1(v,v)-2K_1(u,v)
\end{equation}
\begin{enumerate}
    \item The linear kernel 
    \begin{equation}
        K_1(u,v) = (u^TAv)
    \end{equation}
    where $A$ is a psd matrix. Indeed, 
    \begin{equation}
    P(u,v)=(u-v)^TA(u-v) \leq ||u-v||^2 \sup_{1 \leq i \leq d}\lambda_i    
    \end{equation}
    where $\lambda_i$ are the eigenvalues of $A$ using the Rayleigh quotient property. 
    \item The Gaussian kernel:
    \begin{equation}
        K_1(u,v)=\exp\left(-\frac{1}{2}((u-v)^TA (u-v))\right)
    \end{equation}
    where $A$ is a psd matrix. Indeed, 
    \begin{equation}
        P(u,v)=2-2\exp\left(-\frac{1}{2}((u-v)^TA (u-v))\right)\leq 2(u-v)^TA(u-v) \leq 2 ||u-v||^2 \sup_{1 \leq I \leq d}\lambda_i 
    \end{equation}
    where $\lambda_i$ are the eigenvalues of $A$ and the first inequality comes from the basic inequality $e^x\geq 1+x$
    \item The rational quadratic kernel: 
    \begin{equation}
        K_1(x,y)=\frac{||x-y||^2}{||x-y||^2+\theta}, \theta > 0
    \end{equation}
    Note that in this case, 
    \begin{equation}
        P(u,v) \leq \frac{1}{\theta}||u-v||^2
    \end{equation}
    \item The sinc kernel
    \begin{equation}
    K_1(u,v)=\prod_{i=1}^d \frac{sin(||u_i-v_i||)}{||u_i-v_i||}    
    \end{equation}
    We use the fact that $K_1$ is a translation invariant kernel with associated density $q(y)=\prod_{i=1}^dq_1(y_i)$ with 
    \begin{equation}
        q_1(z)=\frac{1}{2} \mbox{ for } -1 \leq z \leq 1
    \end{equation}
    \item The Mattern kernel with $p > 3/2$. This kernel is translation invariant with associated density $q(y)=\prod_{i=1}^dq_1(y_i)$ with 
    \begin{equation}
        q_1(z) = \frac{1}{(1+x^2)^p}
    \end{equation}
    and 
    \begin{equation}
        E[X^2]<\infty, X \sim q_1
    \end{equation}
\end{enumerate}
\section{An example of a non-autonomous system}
{We provide in this appendix a toy example of a non-autonomous system, namely the harmonic oscillator with sinusoidal input force
\begin{equation}
    \ddot y +0.001\dot y + 10000y = cos(t)
\end{equation}
The kernel is an explicit Fourier random feature kernel with $p=200$ random features as well as a constant term, where time was included as input together with the spatial variables. Each feature was centered and standardized using the training set only for computing the mean and standard deviation. The functions in the corresponding RKHS are then  
\begin{equation}
f([x_1,x_2,t]) = \left[\begin{array}{c}
\sum_{i=1}^{p}\alpha_{i}\cos([z_{1i},z_{2i},z_{3,i}]\cdot [x_1,x_2,t]) + \beta_{i}\sin([z_{1i},z_{2i},z_{3,i}]\cdot [x_1,x_2,t] + \omega_1) \\
\sum_{i=1}^{p}\gamma_{i}\cos([z_{1i},z_{2i},z_{3,i}]\cdot [x_1,x_2,t]) + \delta_{i}\sin([z_{1i},z_{2i},z_{3,i}]\cdot [x_1,x_2,t] + \omega_2 )
\end{array}\right]
\end{equation}
Where the $z$ variables are iid sampled from a standard Normal (or Gaussian) distribution and the parameters $\{\alpha_i,\beta_i,\gamma_i,\delta_i\}, i=1 \ldots p$ along with $\{\omega_j\}, j=1,2$ are learned from the training set. 
Figure \ref{fig:non_auto_plots} illustrates the output ODE-RKHS algorithm for this system. } 
\begin{figure}[H]
    \centering
    \subfigure[3-D Non-autonomous]
    {
    \includegraphics[width=.5\textwidth]{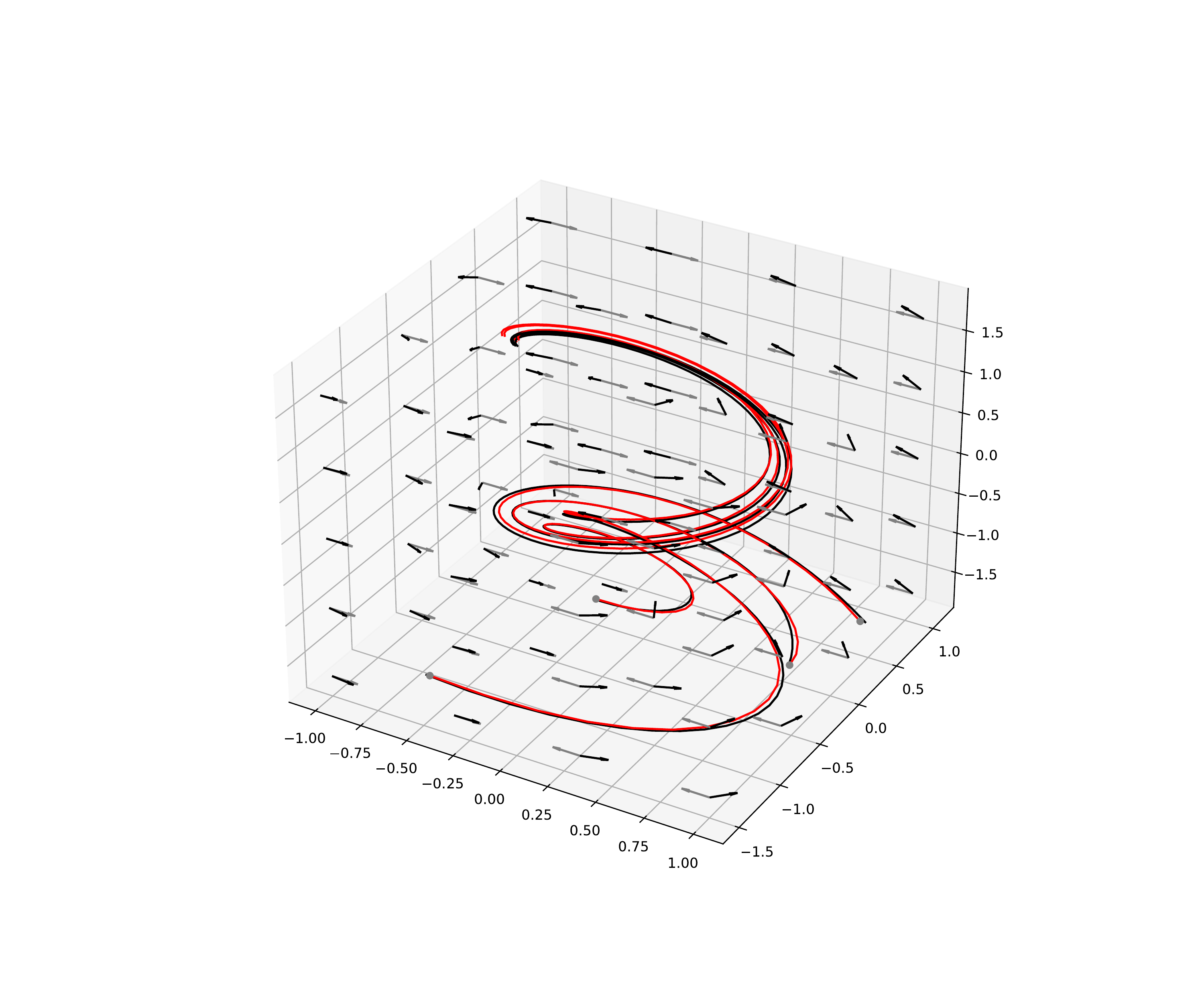}
    }
    \subfigure[1-D Solution curves]{
    \includegraphics[width=.4\textwidth]{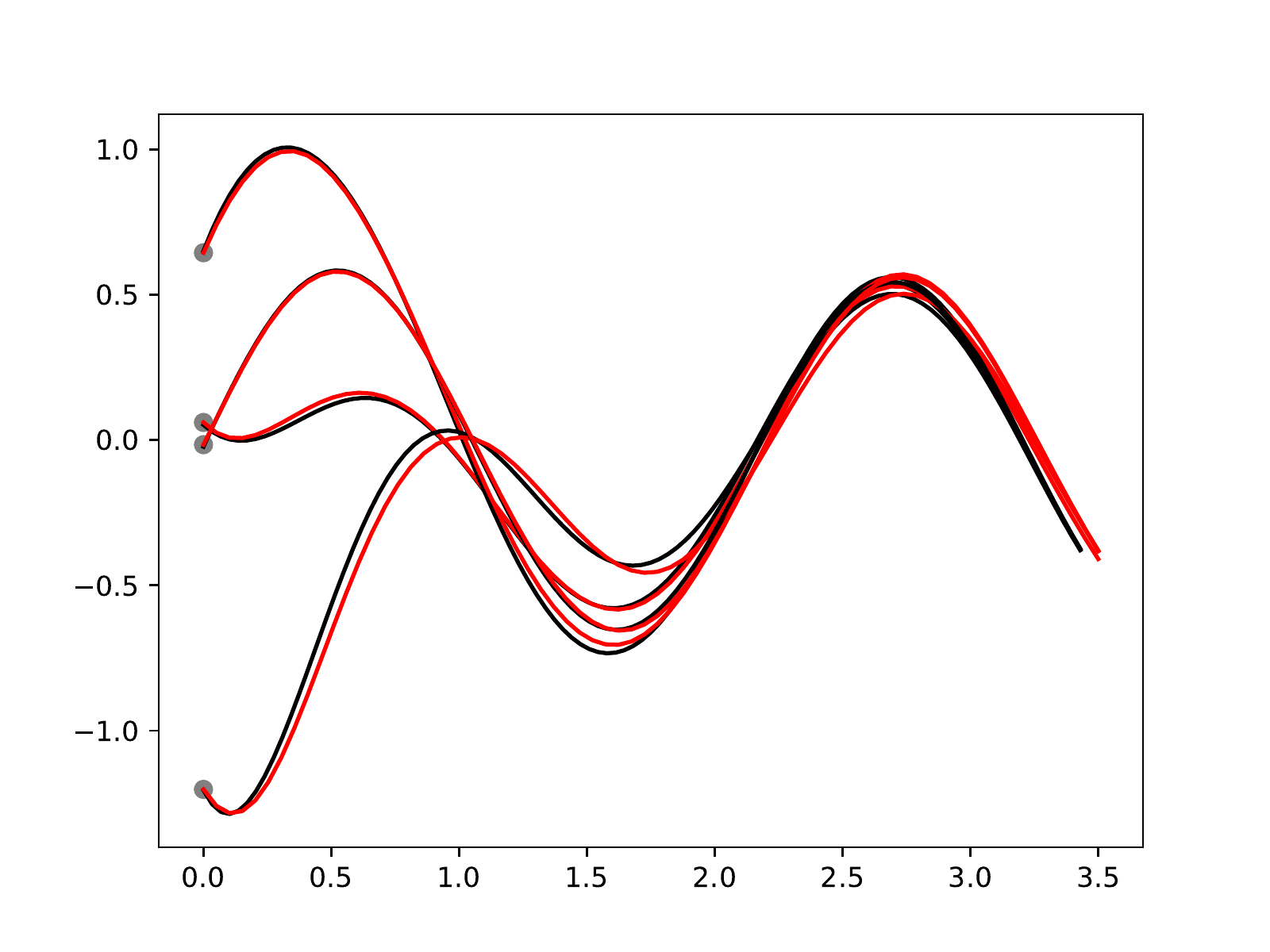}}
    \caption{{ (a): plot of the 2D system where the $z$-axis is time. Black arrows: true vector field. Grey arrows: estimated vector field. Black curves: true trajectories. Red curves: estimated trajectories. (b): Grey points: initial conditions. Black curves: true trajectories. Red curves: estimated trajectories.}}
    \label{fig:non_auto_plots}
\end{figure}

\bibliographystyle{model1-num-names}
\bibliography{MainArxiv111223.bib}

\begin{thebibliography}{38}
\expandafter\ifx\csname natexlab\endcsname\relax\def\natexlab#1{#1}\fi
\providecommand{\url}[1]{\texttt{#1}}
\providecommand{\href}[2]{#2}
\providecommand{\path}[1]{#1}
\providecommand{\DOIprefix}{doi:}
\providecommand{\ArXivprefix}{arXiv:}
\providecommand{\URLprefix}{URL: }
\providecommand{\Pubmedprefix}{pmid:}
\providecommand{\doi}[1]{\href{http://dx.doi.org/#1}{\path{#1}}}
\providecommand{\Pubmed}[1]{\href{pmid:#1}{\path{#1}}}
\providecommand{\bibinfo}[2]{#2}
\ifx\xfnm\relax \def\xfnm[#1]{\unskip,\space#1}\fi
\bibitem[{Hirsch et~al.(2012)Hirsch, Smale, and
  Devaney}]{hirsch2012differential}
\bibinfo{author}{M.~W. Hirsch}, \bibinfo{author}{S.~Smale},
  \bibinfo{author}{R.~L. Devaney}, \bibinfo{title}{Differential equations,
  dynamical systems, and an introduction to chaos},
  \bibinfo{publisher}{Academic press}, \bibinfo{year}{2012}.
\bibitem[{Manton et~al.(2015)Manton, Amblard et~al.}]{manton2015primer}
\bibinfo{author}{J.~H. Manton}, \bibinfo{author}{P.-O. Amblard}, et~al.,
\newblock \bibinfo{title}{A primer on reproducing kernel hilbert spaces},
\newblock \bibinfo{journal}{Foundations and Trends{\textregistered} in Signal
  Processing} \bibinfo{volume}{8} (\bibinfo{year}{2015})
  \bibinfo{pages}{1--126}.
\bibitem[{Dondelinger et~al.(2013)Dondelinger, Husmeier, Rogers, and
  Filippone}]{dondelinger2013ode}
\bibinfo{author}{F.~Dondelinger}, \bibinfo{author}{D.~Husmeier},
  \bibinfo{author}{S.~Rogers}, \bibinfo{author}{M.~Filippone},
\newblock \bibinfo{title}{Ode parameter inference using adaptive gradient
  matching with gaussian processes},
\newblock in: \bibinfo{booktitle}{Artificial intelligence and statistics},
  \bibinfo{organization}{PMLR}, \bibinfo{year}{2013}, pp.
  \bibinfo{pages}{216--228}.
\bibitem[{Brunton et~al.(2016)Brunton, Proctor, and
  Kutz}]{brunton2016discovering}
\bibinfo{author}{S.~L. Brunton}, \bibinfo{author}{J.~L. Proctor},
  \bibinfo{author}{J.~N. Kutz},
\newblock \bibinfo{title}{Discovering governing equations from data by sparse
  identification of nonlinear dynamical systems},
\newblock \bibinfo{journal}{Proceedings of the national academy of sciences}
  \bibinfo{volume}{113} (\bibinfo{year}{2016}) \bibinfo{pages}{3932--3937}.
\bibitem[{Niu et~al.(2016)Niu, Rogers, Filippone, and Husmeier}]{niu2016fast}
\bibinfo{author}{M.~Niu}, \bibinfo{author}{S.~Rogers},
  \bibinfo{author}{M.~Filippone}, \bibinfo{author}{D.~Husmeier},
\newblock \bibinfo{title}{Fast parameter inference in nonlinear dynamical
  systems using iterative gradient matching},
\newblock in: \bibinfo{booktitle}{International Conference on Machine
  Learning}, \bibinfo{organization}{PMLR}, \bibinfo{year}{2016}, pp.
  \bibinfo{pages}{1699--1707}.
\bibitem[{Hu et~al.(2020)Hu, Yang, Zhu, and Hong}]{hu2020revealing}
\bibinfo{author}{P.~Hu}, \bibinfo{author}{W.~Yang}, \bibinfo{author}{Y.~Zhu},
  \bibinfo{author}{L.~Hong},
\newblock \bibinfo{title}{Revealing hidden dynamics from time-series data by
  odenet},
\newblock \bibinfo{journal}{arXiv preprint arXiv:2005.04849}
  (\bibinfo{year}{2020}).
\bibitem[{Qin et~al.(2019)Qin, Wu, and Xiu}]{qin2019data}
\bibinfo{author}{T.~Qin}, \bibinfo{author}{K.~Wu}, \bibinfo{author}{D.~Xiu},
\newblock \bibinfo{title}{Data driven governing equations approximation using
  deep neural networks},
\newblock \bibinfo{journal}{Journal of Computational Physics}
  \bibinfo{volume}{395} (\bibinfo{year}{2019}) \bibinfo{pages}{620--635}.
\bibitem[{Chen et~al.(2018)Chen, Rubanova, Bettencourt, and
  Duvenaud}]{chen2018neural}
\bibinfo{author}{R.~T. Chen}, \bibinfo{author}{Y.~Rubanova},
  \bibinfo{author}{J.~Bettencourt}, \bibinfo{author}{D.~Duvenaud},
\newblock \bibinfo{title}{Neural ordinary differential equations},
\newblock \bibinfo{journal}{arXiv preprint arXiv:1806.07366}
  (\bibinfo{year}{2018}).
\bibitem[{Koopman(1931)}]{koopman1931hamiltonian}
\bibinfo{author}{B.~O. Koopman},
\newblock \bibinfo{title}{Hamiltonian systems and transformation in hilbert
  space},
\newblock \bibinfo{journal}{Proceedings of the national academy of sciences of
  the united states of america} \bibinfo{volume}{17} (\bibinfo{year}{1931})
  \bibinfo{pages}{315}.
\bibitem[{Schmid(2010)}]{schmid2010dynamic}
\bibinfo{author}{P.~J. Schmid},
\newblock \bibinfo{title}{Dynamic mode decomposition of numerical and
  experimental data},
\newblock \bibinfo{journal}{Journal of fluid mechanics} \bibinfo{volume}{656}
  (\bibinfo{year}{2010}) \bibinfo{pages}{5--28}.
\bibitem[{Dai and Li(2022)}]{dai2022kernel}
\bibinfo{author}{X.~Dai}, \bibinfo{author}{L.~Li},
\newblock \bibinfo{title}{Kernel ordinary differential equations},
\newblock \bibinfo{journal}{Journal of the American Statistical Association}
  \bibinfo{volume}{117} (\bibinfo{year}{2022}) \bibinfo{pages}{1711--1725}.
\bibitem[{Heinonen et~al.(2018)Heinonen, Yildiz, Mannerstr{\"o}m, Intosalmi,
  and L{\"a}hdesm{\"a}ki}]{heinonen2018learning}
\bibinfo{author}{M.~Heinonen}, \bibinfo{author}{C.~Yildiz},
  \bibinfo{author}{H.~Mannerstr{\"o}m}, \bibinfo{author}{J.~Intosalmi},
  \bibinfo{author}{H.~L{\"a}hdesm{\"a}ki},
\newblock \bibinfo{title}{Learning unknown ode models with gaussian processes},
\newblock in: \bibinfo{booktitle}{International Conference on Machine
  Learning}, \bibinfo{organization}{PMLR}, \bibinfo{year}{2018}, pp.
  \bibinfo{pages}{1959--1968}.
\bibitem[{Kanagawa et~al.(2018)Kanagawa, Hennig, Sejdinovic, and
  Sriperumbudur}]{kanagawa2018gaussian}
\bibinfo{author}{M.~Kanagawa}, \bibinfo{author}{P.~Hennig},
  \bibinfo{author}{D.~Sejdinovic}, \bibinfo{author}{B.~K. Sriperumbudur},
\newblock \bibinfo{title}{Gaussian processes and kernel methods: A review on
  connections and equivalences},
\newblock \bibinfo{journal}{arXiv preprint arXiv:1807.02582}
  (\bibinfo{year}{2018}).
\bibitem[{Hofmann et~al.(2008)Hofmann, Sch{\"o}lkopf, and
  Smola}]{hofmann2008kernel}
\bibinfo{author}{T.~Hofmann}, \bibinfo{author}{B.~Sch{\"o}lkopf},
  \bibinfo{author}{A.~J. Smola},
\newblock \bibinfo{title}{Kernel methods in machine learning},
\newblock \bibinfo{journal}{The annals of statistics}  (\bibinfo{year}{2008})
  \bibinfo{pages}{1171--1220}.
\bibitem[{Alvarez et~al.(2011)Alvarez, Rosasco, and
  Lawrence}]{alvarez2011kernels}
\bibinfo{author}{M.~A. Alvarez}, \bibinfo{author}{L.~Rosasco},
  \bibinfo{author}{N.~D. Lawrence},
\newblock \bibinfo{title}{Kernels for vector-valued functions: A review},
\newblock \bibinfo{journal}{arXiv preprint arXiv:1106.6251}
  (\bibinfo{year}{2011}).
\bibitem[{Simmons(2016)}]{simmons2016differential}
\bibinfo{author}{G.~F. Simmons}, \bibinfo{title}{Differential equations with
  applications and historical notes}, \bibinfo{publisher}{CRC Press},
  \bibinfo{year}{2016}. \bibinfo{note}{Theorem B}.
\bibitem[{Raissi et~al.(2019)Raissi, Perdikaris, and
  Karniadakis}]{RAISSI2019686}
\bibinfo{author}{M.~Raissi}, \bibinfo{author}{P.~Perdikaris},
  \bibinfo{author}{G.~Karniadakis},
\newblock \bibinfo{title}{Physics-informed neural networks: A deep learning
  framework for solving forward and inverse problems involving nonlinear
  partial differential equations},
\newblock \bibinfo{journal}{Journal of Computational Physics}
  \bibinfo{volume}{378} (\bibinfo{year}{2019}) \bibinfo{pages}{686--707}.
\bibitem[{Dao et~al.(2017)Dao, De~Sa, and R{\'e}}]{dao2017gaussian}
\bibinfo{author}{T.~Dao}, \bibinfo{author}{C.~De~Sa},
  \bibinfo{author}{C.~R{\'e}},
\newblock \bibinfo{title}{Gaussian quadrature for kernel features},
\newblock \bibinfo{journal}{Advances in neural information processing systems}
  \bibinfo{volume}{30} (\bibinfo{year}{2017}) \bibinfo{pages}{6109}.
\bibitem[{Kili{\c{c}} and Stanica(2013)}]{kilicc2013inverse}
\bibinfo{author}{E.~Kili{\c{c}}}, \bibinfo{author}{P.~Stanica},
\newblock \bibinfo{title}{The inverse of banded matrices},
\newblock \bibinfo{journal}{Journal of Computational and Applied Mathematics}
  \bibinfo{volume}{237} (\bibinfo{year}{2013}) \bibinfo{pages}{126--135}.
\bibitem[{Vershynin(2018)}]{vershynin2018high}
\bibinfo{author}{R.~Vershynin}, \bibinfo{title}{High-dimensional probability:
  An introduction with applications in data science},
  volume~\bibinfo{volume}{47}, \bibinfo{publisher}{Cambridge university press},
  \bibinfo{year}{2018}.
\bibitem[{Quinonero-Candela and Rasmussen(2005)}]{quinonero2005unifying}
\bibinfo{author}{J.~Quinonero-Candela}, \bibinfo{author}{C.~E. Rasmussen},
\newblock \bibinfo{title}{A unifying view of sparse approximate gaussian
  process regression},
\newblock \bibinfo{journal}{The Journal of Machine Learning Research}
  \bibinfo{volume}{6} (\bibinfo{year}{2005}) \bibinfo{pages}{1939--1959}.
\bibitem[{Kokotovic and Heller(1967)}]{kokotovic1967direct}
\bibinfo{author}{P.~Kokotovic}, \bibinfo{author}{J.~Heller},
\newblock \bibinfo{title}{Direct and adjoint sensitivity equations for
  parameter optimization},
\newblock \bibinfo{journal}{IEEE Transactions on Automatic Control}
  \bibinfo{volume}{12} (\bibinfo{year}{1967}) \bibinfo{pages}{609--610}.
\bibitem[{Zheng et~al.(2018)Zheng, Askham, Brunton, Kutz, and
  Aravkin}]{zheng2018unified}
\bibinfo{author}{P.~Zheng}, \bibinfo{author}{T.~Askham}, \bibinfo{author}{S.~L.
  Brunton}, \bibinfo{author}{J.~N. Kutz}, \bibinfo{author}{A.~Y. Aravkin},
\newblock \bibinfo{title}{A unified framework for sparse relaxed regularized
  regression: Sr3},
\newblock \bibinfo{journal}{IEEE Access} \bibinfo{volume}{7}
  (\bibinfo{year}{2018}) \bibinfo{pages}{1404--1423}.
\bibitem[{de~Silva et~al.(2020)de~Silva, Champion, Quade, Loiseau, Kutz, and
  Brunton}]{desilva2020}
\bibinfo{author}{B.~de~Silva}, \bibinfo{author}{K.~Champion},
  \bibinfo{author}{M.~Quade}, \bibinfo{author}{J.-C. Loiseau},
  \bibinfo{author}{J.~Kutz}, \bibinfo{author}{S.~Brunton},
\newblock \bibinfo{title}{Pysindy: A python package for the sparse
  identification of nonlinear dynamical systems from data},
\newblock \bibinfo{journal}{Journal of Open Source Software}
  \bibinfo{volume}{5} (\bibinfo{year}{2020}) \bibinfo{pages}{2104}.
\bibitem[{Lew et~al.(2023)Lew, Hekal, Potomkin, Kochdumper, andG Stanley~Bak,
  and Bogomolov}]{Lew2023}
\bibinfo{author}{E.~Lew}, \bibinfo{author}{A.~Hekal},
  \bibinfo{author}{K.~Potomkin}, \bibinfo{author}{N.~Kochdumper},
  \bibinfo{author}{B.~H. andG Stanley~Bak}, \bibinfo{author}{S.~Bogomolov},
\newblock \bibinfo{title}{Autokoopman: A toolbox for automated system
  identification via koopman operator linearization},
\newblock in: \bibinfo{booktitle}{International Conference on Computer Aided
  Verification (CAV)}, \bibinfo{year}{2023}. \bibinfo{note}{Under review}.
\bibitem[{Williams et~al.(2015)Williams, Kevrekidis, and
  Rowley}]{williams2015data}
\bibinfo{author}{M.~O. Williams}, \bibinfo{author}{I.~G. Kevrekidis},
  \bibinfo{author}{C.~W. Rowley},
\newblock \bibinfo{title}{A data--driven approximation of the koopman operator:
  Extending dynamic mode decomposition},
\newblock \bibinfo{journal}{Journal of Nonlinear Science} \bibinfo{volume}{25}
  (\bibinfo{year}{2015}) \bibinfo{pages}{1307--1346}.
\bibitem[{Yeung et~al.(2019)Yeung, Kundu, and Hodas}]{yeung2019learning}
\bibinfo{author}{E.~Yeung}, \bibinfo{author}{S.~Kundu},
  \bibinfo{author}{N.~Hodas},
\newblock \bibinfo{title}{Learning deep neural network representations for
  koopman operators of nonlinear dynamical systems},
\newblock in: \bibinfo{booktitle}{2019 American Control Conference (ACC)},
  \bibinfo{organization}{IEEE}, \bibinfo{year}{2019}, pp.
  \bibinfo{pages}{4832--4839}.
\bibitem[{DeGennaro and Urban(2019)}]{degennaro2019scalable}
\bibinfo{author}{A.~M. DeGennaro}, \bibinfo{author}{N.~M. Urban},
\newblock \bibinfo{title}{Scalable extended dynamic mode decomposition using
  random kernel approximation},
\newblock \bibinfo{journal}{SIAM Journal on Scientific Computing}
  \bibinfo{volume}{41} (\bibinfo{year}{2019}) \bibinfo{pages}{A1482--A1499}.
\bibitem[{Zhao and Giannakis(2016)}]{zhao2016analog}
\bibinfo{author}{Z.~Zhao}, \bibinfo{author}{D.~Giannakis},
\newblock \bibinfo{title}{Analog forecasting with dynamics-adapted kernels},
\newblock \bibinfo{journal}{Nonlinearity} \bibinfo{volume}{29}
  (\bibinfo{year}{2016}) \bibinfo{pages}{2888}.
\bibitem[{Burov et~al.(2021)Burov, Giannakis, Manohar, and
  Stuart}]{burov2021kernel}
\bibinfo{author}{D.~Burov}, \bibinfo{author}{D.~Giannakis},
  \bibinfo{author}{K.~Manohar}, \bibinfo{author}{A.~Stuart},
\newblock \bibinfo{title}{Kernel analog forecasting: Multiscale test problems},
\newblock \bibinfo{journal}{Multiscale Modeling \& Simulation}
  \bibinfo{volume}{19} (\bibinfo{year}{2021}) \bibinfo{pages}{1011--1040}.
\bibitem[{Schaeffer et~al.(2020)Schaeffer, Tran, Ward, and
  Zhang}]{schaeffer2020extracting}
\bibinfo{author}{H.~Schaeffer}, \bibinfo{author}{G.~Tran},
  \bibinfo{author}{R.~Ward}, \bibinfo{author}{L.~Zhang},
\newblock \bibinfo{title}{Extracting structured dynamical systems using sparse
  optimization with very few samples},
\newblock \bibinfo{journal}{Multiscale Modeling \& Simulation}
  \bibinfo{volume}{18} (\bibinfo{year}{2020}) \bibinfo{pages}{1435--1461}.
\bibitem[{Pontryagin(1987)}]{pontryagin1987mathematical}
\bibinfo{author}{L.~S. Pontryagin}, \bibinfo{title}{Mathematical theory of
  optimal processes}, \bibinfo{publisher}{CRC press}, \bibinfo{year}{1987}.
\bibitem[{Younes(2020)}]{JMLR:v21:18-415}
\bibinfo{author}{L.~Younes},
\newblock \bibinfo{title}{Diffeomorphic learning},
\newblock \bibinfo{journal}{Journal of Machine Learning Research}
  \bibinfo{volume}{21} (\bibinfo{year}{2020}) \bibinfo{pages}{1--28}.
\bibitem[{Lorenz(1995)}]{75462}
\bibinfo{author}{E.~Lorenz}, \bibinfo{title}{Predictability: a problem partly
  solved}, Ph.D. thesis, \bibinfo{address}{Shinfield Park, Reading},
  \bibinfo{year}{1995}.
\bibitem[{Murphy and LeVine~III(2010)}]{murphy2010alzheimer}
\bibinfo{author}{M.~P. Murphy}, \bibinfo{author}{H.~LeVine~III},
\newblock \bibinfo{title}{Alzheimer's disease and the amyloid-$\beta$ peptide},
\newblock \bibinfo{journal}{Journal of Alzheimer's disease}
  \bibinfo{volume}{19} (\bibinfo{year}{2010}) \bibinfo{pages}{311--323}.
\bibitem[{Vernhet et~al.(2020)Vernhet, Bilgel, Durrleman, Resnick, Johnson, and
  Jedynak}]{vernhet2020modeling}
\bibinfo{author}{P.~Vernhet}, \bibinfo{author}{M.~Bilgel},
  \bibinfo{author}{S.~Durrleman}, \bibinfo{author}{S.~M. Resnick},
  \bibinfo{author}{S.~C. Johnson}, \bibinfo{author}{B.~M. Jedynak},
\newblock \bibinfo{title}{Modeling the early accumulation of amyloid using
  differential equations in wrap and blsa: Neuroimaging/optimal neuroimaging
  measures for early detection},
\newblock \bibinfo{journal}{Alzheimer's \& Dementia} \bibinfo{volume}{16}
  (\bibinfo{year}{2020}) \bibinfo{pages}{e039536}.
\bibitem[{Doob(1940)}]{doob1940regularity}
\bibinfo{author}{J.~L. Doob},
\newblock \bibinfo{title}{Regularity properties of certain families of chance
  variables},
\newblock \bibinfo{journal}{Transactions of the American Mathematical Society}
  \bibinfo{volume}{47} (\bibinfo{year}{1940}) \bibinfo{pages}{455--486}.
\bibitem[{Atkinson(2008)}]{atkinson2008introduction}
\bibinfo{author}{K.~E. Atkinson}, \bibinfo{title}{An introduction to numerical
  analysis}, \bibinfo{publisher}{John wiley \& sons}, \bibinfo{year}{2008}.

\end{thebibliography}

\end{document}